\documentclass[onecolumn,draftcls]{IEEEtran}
\usepackage{amsfonts}
\usepackage{times}		% Use Times font
\usepackage{amsmath}
\usepackage{amsthm}
\usepackage[linesnumbered,ruled,vlined]{algorithm2e}
\usepackage{multicol}
\usepackage{graphicx}
\usepackage{subfigure}
\usepackage{colortbl,booktabs}
\usepackage{threeparttable}
\usepackage{subeqnarray}
\usepackage{cases}
\usepackage{multicol}
\usepackage{multirow}
\DeclareMathOperator*{\argmin}{argmin}

\newcommand{\dhline}{\hline\hline}

\newtheorem*{thmA}{Theorem 1}
\newtheorem{theorem}{Theorem}
\newtheorem{lemma}{Lemma}
\usepackage{ifpdf}
 \ifpdf
   \pdfoutput=1
 \else
 \fi
\ifCLASSOPTIONcompsoc
\else
\fi
\ifCLASSINFOpdf
\else
\fi
\begin{document}
\title{Distortion-driven Turbulence Effect Removal using Variational Model}

\author{Yuan~Xie,~\IEEEmembership{Member,~IEEE,}
        Wensheng~Zhang,
        Dacheng~Tao,~\IEEEmembership{Senior~Member,~IEEE},
        Wenrui Hu,
        Yanyun~Qu,
        and Hanzi~Wang,~\IEEEmembership{Senior~Member,~IEEE,}
\IEEEcompsocitemizethanks{
\IEEEcompsocthanksitem Y. Xie, W. Zhang and W. Hu are with the State Key Lab. of Complex Systems and Intelligence Science, Institute of Automation, Chinese Academy of Sciences, Beijing, 100190, China; E-mail: \{yuan.xie, wensheng.zhang, wenrui.hu\}@ia.ac.cn
\IEEEcompsocthanksitem D. Tao is with the Center for Quantum Computation \& Intelligent Systems and the Faculty of Engineering \& Information Technology, University of Technology, Sydney, Australia; E-mail: dacheng.tao@uts.edu.au
\IEEEcompsocthanksitem Y. Qu and H. Wang are with School of Information Science and Technology, Xiamen University, Fujian, 361005, China; E-mail: \{yyqu, hanzi.wang\}@xmu.edu.cn}% <-this % stops a space
%\IEEEcompsocthanksitem D. Tao is with the Center for Quantum Computation \& Intelligent Systems and the Faculty of Engineering \& Information Technology, University of Technology, Sydney, Australia; E-mail: dacheng.tao@uts.edu.au}% <-this % stops a space
\thanks{}}


\markboth{SUBMIT TO IEEE TRANSACTIONS ON PATTERN ANALYSIS AND MACHINE INTELLIGENCE}%
{Shell \MakeLowercase{\textit{et al.}}: Bare Demo of IEEEtran.cls for Computer Society Journals}

%\IEEEpubid{\makebox[\columnwidth]{\hfill 0000--0000/00/\$00.00~\copyright~2007 IEEE}%

% PRIOR to the title within the \IEEEcompsoctitleabstractindextext IEEEtran
\IEEEcompsoctitleabstractindextext{%
\begin{abstract}
%\boldmath
It remains a challenge to simultaneously remove geometric distortion and space-time-varying blur in frames captured through a turbulent atmospheric medium. To solve, or at least reduce these effects, we propose a new scheme to recover a latent image from observed frames by integrating a new variational model and distortion-driven spatial-temporal kernel regression. The proposed scheme first constructs a high-quality reference image from the observed frames using low-rank decomposition. Then, to generate an improved registered sequence, the reference image is iteratively optimized using a variational model containing a new spatial-temporal regularization. The proposed fast algorithm efficiently solves this model without the use of partial differential equations (PDEs). Next, to reduce blur variation, distortion-driven spatial-temporal kernel regression is carried out to fuse the registered sequence into one image by introducing the concept of the near-stationary patch. Applying a blind deconvolution algorithm to the fused image produces the final output. Extensive experimental testing shows, both qualitatively and quantitatively, that the proposed method can effectively alleviate distortion and blur and recover details of the original scene compared to state-of-the-art methods.
\end{abstract}
% not want either math or citations to appear in the abstract.

% Note that keywords are not normally used for peer review papers.
}

\maketitle

% is not selected <OR> if conference mode is selected - because compsoc
% conference papers position the abstract like regular (non-compsoc)
\IEEEdisplaynotcompsoctitleabstractindextext

% \fi
\IEEEpeerreviewmaketitle

\section{Introduction}
% Computer Society journal papers do something a tad strange with the very
% first section heading (almost always called "Introduction"). They place it
% ABOVE the main text! IEEEtran.cls currently does not do this for you.
% However, You can achieve this effect by making LaTeX jump through some
% hoops via something like:
%
%\ifCLASSOPTIONcompsoc
%  \noindent\raisebox{2\baselineskip}[0pt][0pt]%
%  {\parbox{\columnwidth}{\section{Introduction}\label{sec:introduction}%
%  \global\everypar=\everypar}}%
%  \vspace{-1\baselineskip}\vspace{-\parskip}\par
%\else
%  \section{Introduction}\label{sec:introduction}\par
%\fi
%
% Admittedly, this is a hack and may well be fragile, but seems to do the
% trick for me. Note the need to keep any \label that may be used right
% after \section in the above as the hack puts \section within a raised box.


\begin{figure}[!htbp]
\setlength{\abovecaptionskip}{0pt}  %±êÌâµÄÉϱ߽ç
\setlength{\belowcaptionskip}{0pt} %±êÌâµÄϱ߽ç
\renewcommand{\figurename}{\footnotesize{Figure}}
\centering
\includegraphics[width=0.6\textwidth]{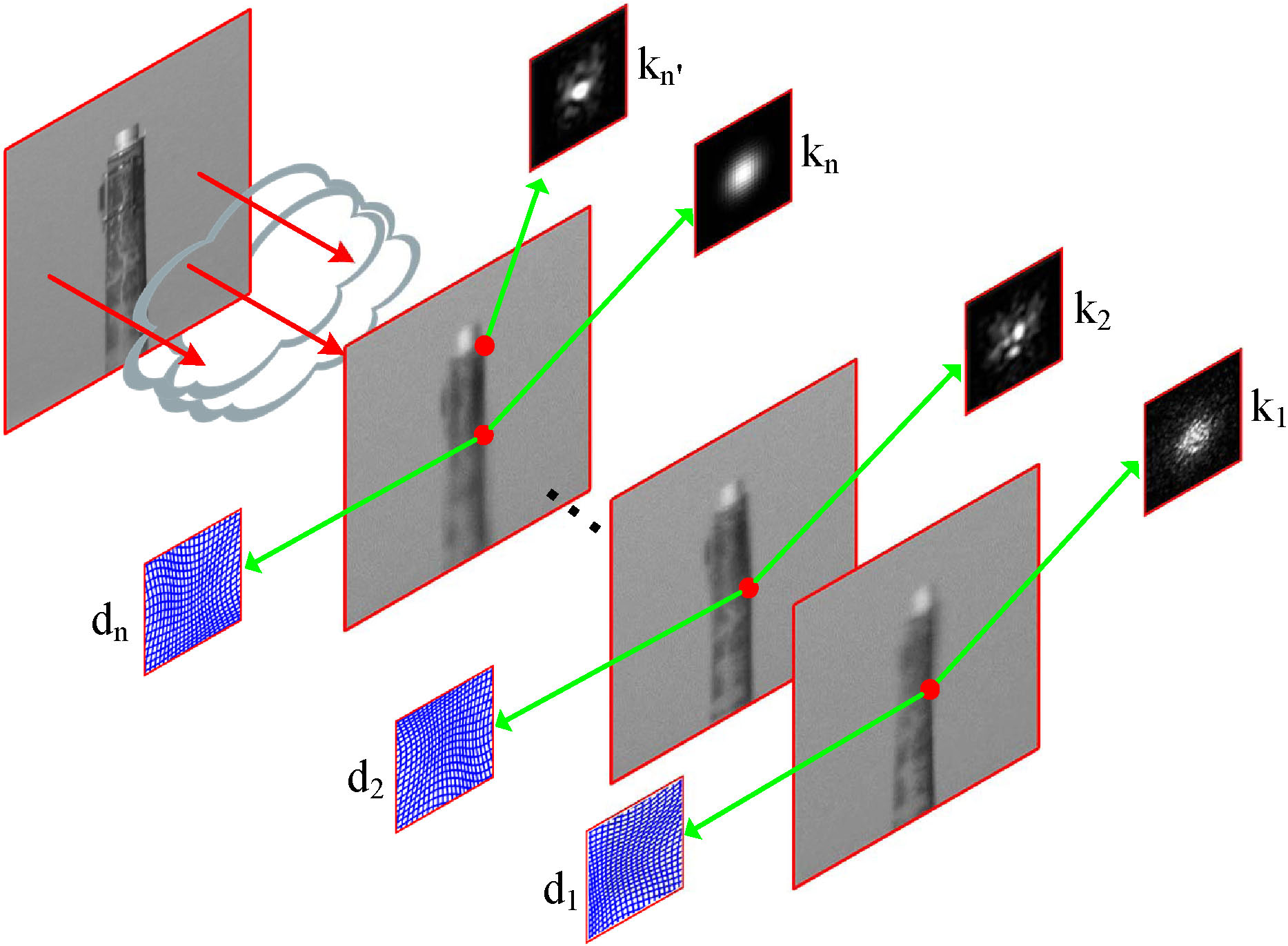}
\caption{\footnotesize{Illustration of the visual effects caused by atmospheric turbulence. $d_1$ to $d_n$ are local deformation fields, which means that a distortion effect exists. Blur kernels $k_1$ to $k_n$ and $k_n'$ indicate the space-time-varying blur effect.}}
\label{fig:visual-effect}
\end{figure}
\IEEEPARstart{A}{tmospheric} turbulence can severely degrade the quality of images produced by long range observation systems, rendering the images unsuitable for vision applications such as surveillance or scene inspection. The main visual effects caused by atmospheric turbulence are geometric distortion and space-time-varying blur (see examples in Fig. \ref{fig:visual-effect}). The distortion is primarily generated by (1) optical turbulence and (2) scattering and absorption by particulates; aerosols, for example, diffuse light will also cause blur \cite{optics1}. Several approaches have been used to restore images, including adaptive optics techniques ({\it e.g.} \cite{optics2,optics3}) and pure image processing-based methods (such as \cite{vorontsov1,vorontsov2,xzhu1,xzhu2,hotair}). Due to random fluctuations in turbulence, calculating a reasonable estimation of atmospheric modulation transfer function (MTF) is extremely difficult. However, this function is of critical importance for optics-based restoration methods. Therefore, in this article, we only focus on using image processing to handle the degradation caused by turbulence. Supposing that the scene and the image sensor are both static, we adopt the mathematical model used in \cite{xzhu2,EFF,base-model} to interpret imaging processing through the turbulence:
\begin{equation}\label{degrede_model}
    f_i(\mathbf{x}) = D_{i,blur(\mathbf{x})} \Big(H(u(\mathbf{x}))\Big) + \varepsilon_i, \quad \forall i, \quad
    i\in [1,\ldots,N]
\end{equation}
where $u$ is the static original scene needed to be retrieved, $f_i$ is the observed frame at time $i$, $N$ represents the number of observed frames, the vector $\mathbf{x} = (x,y)^{T}$ is a $2$-D spatial location, and $\varepsilon_i$ denotes the sensor noise. $H$ is a blurring operator, which is caused by sensor optics, and corresponds to a space-invariant diffraction-limited point spread function (PSF) $h$. $D$ is a deformation operator, and its subscript indicates that both the local deformation and the space-varying blur (the PSF is characterized by $h_{i,x}$) exist concurrently.

Simultaneously removing space-time-varying blur and distortion is a nontrivial problem. Li {\it et al.} \cite{atomspheric-pca} explicitly formulated multichannel image deconvolution as a principal component analysis (PCA) problem, and applied it to the restoration of atmospheric turbulence-degraded images. However, this spectral method does not fully correct the deformation. Moreover, due to the fact that high-frequency information is discarded, the local texture of the true scene is also poorly recovered. Hirsch {\it et al.} \cite{EFF} utilized space-varying blind deconvolution to alleviate turbulence distortion. While reasonably effective, it does not take sensor noise into account when estimating the local PSF, which results in deblurring artifacts.

Image selection and fusion methodologies have been deployed to produce a high-quality latent image. The ``lucky frame'' methods \cite{luckyframe1,luckyframe2,luckyregion1,luckyregion2} select the relatively high quality frame from a degraded sequence by using sharpness as the image quality measurement. However, the so-called ``lucky frame'' is unlikely to exist in a short exposure video stream. To alleviate this problem, Aubailly {\it et al.} \cite{luckyregion3} proposed a local version of the ``lucky frame'' method referred to as the ``lucky region'' method. In this method, a high-quality image from the output is fused with many small lucky regions detected using a local sharpness metric. However, even though the space-time-varying blur can be removed during the fusion process, the final output is still susceptible to the blur caused by the diffraction-limited PSF \cite{xzhu2}.

The success of some recently proposed turbulence removal methods stems from the use of diffeomorphic warping and image sharpening techniques \cite{hotair,temporal-mean,D-Frakes,bregman_distort}. Shimizu {\it et al.} \cite{hotair} applied a temporal median filter to build a reference image, and then fixed the geometric distortion using B-spline non-rigid registration associated with an additional stabilization term. A high-resolution latent image is then obtained by employing a super-resolution method. Mao and Gilles \cite{bregman_distort} combined optical flow-based geometric correction with a non-local total variance-based regularization process to recover the original scene. However, this method only focuses on correcting distortion and ignores the rich information in multiple frames when recovering the detail of the image.

Zhu {\it et al.} \cite{xzhu2} proposed the diffeomorphic warping and image sharpening approach by combining a symmetric constraint-based B-spline registration associated with near-diffraction-limited image reconstruction to reduce the space and time-varying problem to a shift invariant one. However, this method has two main limitations: (1) the method uses the temporal mean of the observed frame to calculate the reference image, which leads to a poor registration result (especially in the case of strong turbulence); and (2) the method constructs a single image from the near-diffraction-limited detection based on the assumption that the distortion can be effectively removed, meaning that the noise introduced by registration error cannot be reduced.

In this paper, we propose a new scheme for restoring a single sharp image of the original scene from a sequence of observed frames degraded by atmospheric turbulence. The proposed method can effectively remove both local deformation and spatially-varying blur and recover the image details, even in the case of strong turbulence. The scheme consists of the following four steps. First, we construct a reference image by \textbf{\emph{low-rank decomposition}}, resulting in a sharper and less noisy image than the traditional method. Second, the reference image is iteratively optimized and enhanced using a variational model involving a new \textbf{\emph{spatial-temporal regularization}}, which helps generate an improved registered sequence. Moreover, we design \textbf{\emph{a simple but highly efficient algorithm}} without the use of PDEs in order to solve the variational model, which is supported by a rigorous proof of convergence. Third, by introducing the concept of the near-stationary patch, a single image with reduced space-varying blur is produced using a fusion process. In this step, to avoid noise effects, \textbf{\emph{distortion-driven spatial-temporal kernel regression}} is employed to eliminate the noise both in the image and the temporal domains. To the best of our knowledge, using the information from the deformation field to guide the construction of the local kernel has not previously been described. Finally, a blind space-invariant deconvolution algorithm can be used to generate the final output.

\begin{figure*}[htbp]
\setlength{\abovecaptionskip}{0pt}  %±êÌâµÄÉϱ߽ç
\setlength{\belowcaptionskip}{0pt} %±êÌâµÄϱ߽ç
\renewcommand{\figurename}{Figure}
\centerline{
\begin{minipage}[b]{1.1\textwidth}
\centerline{
\includegraphics[width=\textwidth]{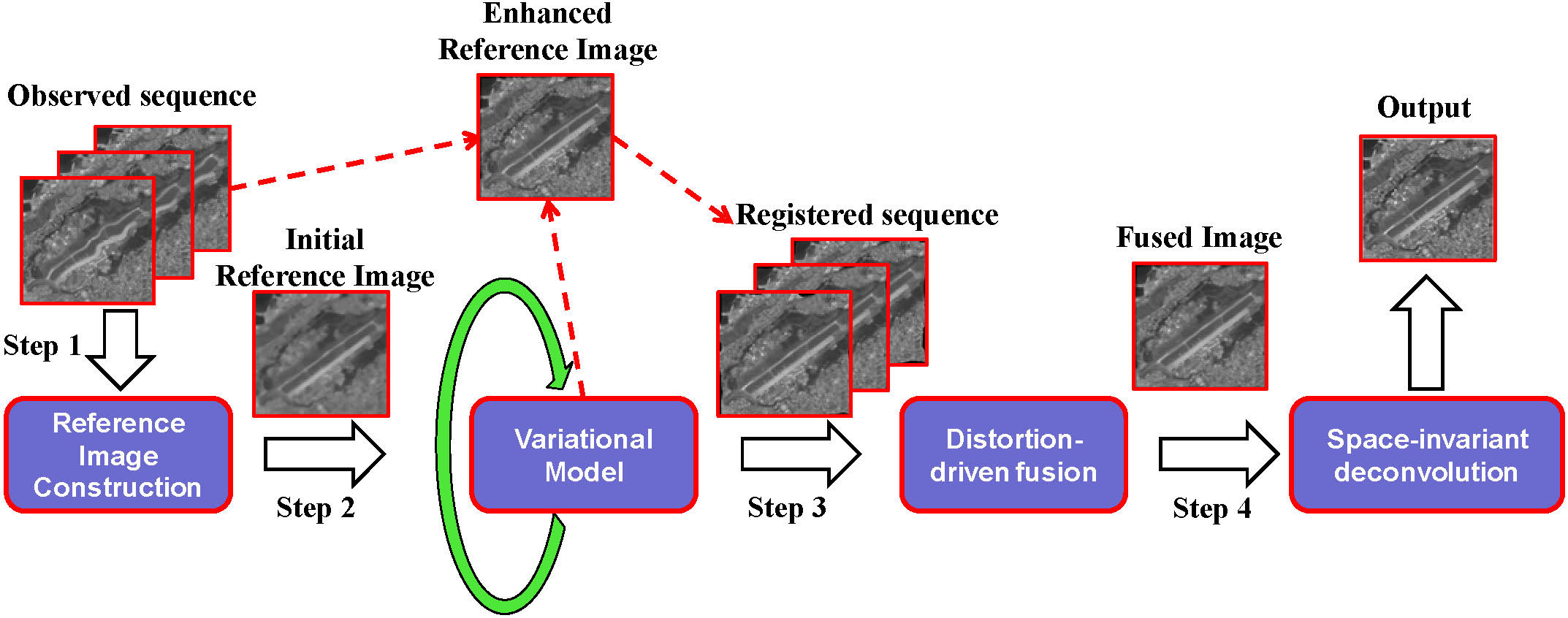}}
\caption{Block diagram of the proposed restoration scheme. (1) Step 1: construct a high-quality reference image from observed sequence using low-rank decomposition; this reference image will be used as the initial value of the variational model in the next step. (2) Step 2: enhance the reference image iteratively using a new variational model (as the green arrow shown). After the convergence of optimization, the observed sequence is registered to the enhanced reference image to produce a registered sequence with little local deformation (as the dashed arrow line indicated). (3) Step 3: a distortion-driven fusion process is employed to fuse the registered sequence to one image with only space-invariant blur effect existing. (4) Step 4: final restored image is produced by a commonly space-invariant deconvolution method.}
\label{fig:overview}
\end{minipage}}
\end{figure*}

This paper is structured as follows. Section \ref{frame_work} provides an overview of the proposed restoration algorithm. In Section \ref{Spatial-Temporal-Regularization}, we detail the method used to optimize the reference image. Image fusion from the registered sequence is presented in Section \ref{image_recon}, while single image deconvolution is presented in Section \ref{final-deconvolution}. Experimental results are presented in Section \ref{experiment} and we discuss the methods and conclude in Section \ref{discussion-and-conclusion}.

\section{Restoration Scheme Overview}\label{frame_work}
The proposed restoration framework has four main steps (see Fig. \ref{fig:overview}). Given an observed sequence $\{f_i\}$, step $1$ applies the low-rank decomposition to $\{f_i\}$ to generate a high-quality reference image. In multi-frame registration, the reference image is usually obtained by computing the temporal mean of the observed frames \cite{xzhu1,xzhu2,hotair,temporal-mean,D-Frakes,bregman_distort,direct-inverse}. An averaged image such as this is always blurry and noisy, especially when turbulence is strong. Inspired by \cite{RPCA}, which utilizes sparse decomposition for background modeling, we use matrix decomposition to obtain the low-rank part and construct the reference image for registration. Our rationale for using this approach is that the low-rank part is a stable component of a "dancing image", and therefore corresponds to the original scene to some extent. The decomposition can be defined as follows:
\begin{equation}\label{low_rank_decomp1}
    \text{minimize   } ||L||_{\ast} + \lambda ||S||_{1} \qquad \text{subject to   } L + S = G
\end{equation}
where $G\in R^{m\times n}$ is the distorted sequence matrix with each column being a distorted frame vector $f_i$, $m$ denotes the total pixels in each frame, and $n$ denotes the number of frames in the sequence. $L\in R^{m\times n}$ is the low-rank component of $G$, $S\in R^{m\times n}$ is the sparse component of $G$, $||\cdot||_{\ast}$ is the nuclear norm defined by the sum of all singular values, $||\cdot||_{1}$ is the $l_1$-norm defined by the component-wise sum of absolute values of all the entries, and $\lambda$ is a constant providing a trade-off between the sparse and low-rank components. Some decomposed results are illustrated in Fig. \ref{fig:low_rank_decomp}. We exploit robust principal component analysis (RPCA) to estimate the reference image, which acts as the initial value for the next step.

Step $2$ enhances the reference image using a variational model employing spatial-temporal regularization. This optimization is the inverse process of the following problem
\begin{equation}\label{deform_model}
    f_i = \Phi_i u + \varepsilon, \quad \forall i \in [1,\ldots,N]
\end{equation}
where $\Phi_i$ is the linear operator corresponding to the deformation field $\overleftarrow{F_i}$ obtained by B-spline registration \cite{ffd,mirt} and $u$ represents the reference image to be enhanced. The variational model optimizes   with only a few rounds of non-rigid registration and subproblem optimization. After the convergence of the optimization procedure, the observed frames $\{f_i\}$ are registered to the enhanced reference image to achieve the registered sequence $\{R_i\}$. This step effectively removes local deformation.

Step $3$ restores a single image $Z$ from the registered sequence $\{R_i\}$ by distortion-driven fusion in order to eliminate the space-time-varying blur. A near-stationary patch can be detected for each local region from the patch sequence through the temporal domain. Spatial-temporal kernel regression is then carried out to reduce the noise caused by optics and registration error.  Fusing all the denoised near-stationary patches generates an image $Z$. While the output $Z$ is still a blurred image, it can be approximately restored using a common spatially invariant deblurring technique.

In the final step, based on the statistical prior of natural images, a single image blind deconvolution algorithm is implemented on $Z$ to further remove blur and enhance image quality. Details of steps $3$, $4$ and $5$ will be given in the following sections.

\begin{figure}[!htbp]
\setlength{\abovecaptionskip}{0pt}  %±êÌâµÄÉϱ߽ç
\setlength{\belowcaptionskip}{0pt} %±êÌâµÄϱ߽ç
\renewcommand{\figurename}{\footnotesize{Figure}}
\centering
\includegraphics[width=0.6\textwidth]{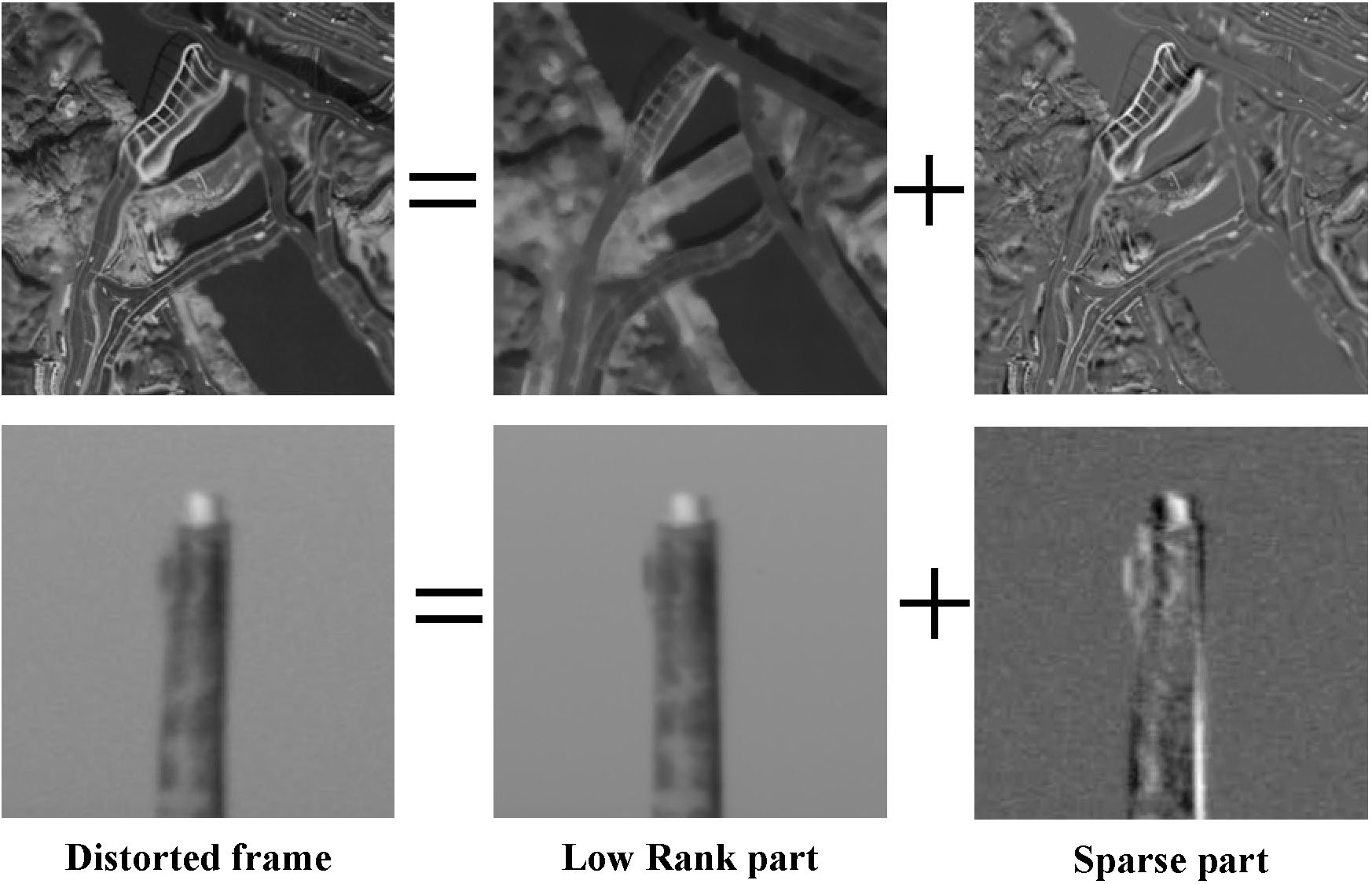}
\caption{\footnotesize{Low-rank decomposition of the distorted frames in city\_strong and chimney sequences}}
\label{fig:low_rank_decomp}
\end{figure}

\section{Optimize the Reference Image Using a Variational Model}\label{Spatial-Temporal-Regularization}
The obtained reference image from RPCA is not completely ready for non-rigid registration because the inner structure of the reference image cannot be easily recovered. In other words, the edges representing the object's profile tend to be well defined, while the edges that characterize the inner structure of the object tend to be intermittent and noisy. Therefore, the reference image needs to be refined. Assuming that the observed image sequence is $\{f_i\}_{i=1}^{N}$ and the reference image that we want to optimize is $u$, then directly reconstructing $u$ from Eq. (\ref{deform_model}) is an ill-posed problem. There is a need to introduce a regularization technique. If we denote the regularization term of the image as $J(u)$, then we can get an unconstrained optimization problem
\begin{align}\label{eqn:res_mixed_model}
    &\min_{u} E(u) +  J(u) \\ \label{concept_opt}
    &= \min_{u} \underbrace{\sum \nolimits_{i} \| \Phi_{i} u - f_{i} \|_{2}^{2}}_{E(u)} + \underbrace{\Big(\mu_1 J_{s}(u) + \mu_2 J_{t}(u)\Big)}_{J(u)} \\
    &= \min_{u} \underbrace{\sum \nolimits_{i} \| \Phi_{i} u - f_{i} \|_{2}^{2}}_{E(u)} + \underbrace{\Big(\mu_1 \overbrace{|u|_{NLTV}}^{J_{s}(u)} + \mu_2 \overbrace{|(u - u_{p})|_{TV}}^{J_{t}(u)} \Big)}_{J(u)} \label{concept_opt1}
\end{align}
where $E(u)$ measures the fidelity of the observed data, and $J(u)$ includes the spatial regularizer $J_s(u)$ and the temporal regularizer $J_t(u)$ to constraint the solution of $u$. $J_s(u)$ is the non-local total variation which can explore repetitive structures to preserve important details of an image while effectively removing artifact. $J_t(u)$ is the standard total variation of the difference between two consecutively optimized results ($u_p$ denotes the result obtained by the previous step), which does not only restrict the smoothness between the iterations but also forces the local energy of the optimized result to converge to a local structure ({\it e.g.} the edge of the object). The parameters $\mu_1$ and $\mu_2$ are chosen as a trade-off between the two regularizers. Because $J(u)$ applies the sparsifying transform in both the spatial and temporal domains, we refer to it as the \textbf{spatial-temporal regularization}. In the following subsections, we develop a fast algorithm to solve the optimization problem (\ref{concept_opt1}).

\subsection{Optimization by Bregman Iteration}
Bregman iteration is a concept that originated from functional analysis and is commonly used to find the extrema of convex functions \cite{Bregman}. It was introduced in \cite{Bregman_Iter} for total variation-based image processing, before being extended to other applications such as wavelet-based denoising \cite{wavelet} and magnetic resonance imaging \cite{MR}. In this subsection, we briefly describe how to restore the optimization problem (\ref{concept_opt1}) using this technique.

The Bregman distance \cite{Bregman} associated with a convex function $J$ between points $u$ and $v$ is
\begin{equation}\label{Bregman_dis}
    D_{J}^{p}(u,v) = J(u) - J(v) - \langle p, u-v \rangle \geq 0, \text{where } p\in \partial J(v)
\end{equation}
where $p$ is the subgradient of $J$ at $v$. Because $D_{J}^{p}(u,v)\neq D_{J}^{p}(v,u)$, $D_{J}^{p}(u,v)$ is not a distance in the usual sense. However, it measures the closeness between $u$ and $v$ in the sense that $D_{J}^{p}(u,v)>0$, and $D_{J}^{p}(u,v)\geq D_{J}^{p}(w,v)$ for any point $w$ being a convex combination of $u$ and $v$. According to \cite{Bregman_Iter}, and using the Bregman distance, the optimization problem (\ref{concept_opt1}) can be solved by the following iteration
\begin{subequations}
    \begin{numcases}{}
        u^{k} =\argmin_{u} \mu J(u) + \sum_{i=1}^{N} \|\Phi_i u - f_{i}^{k}\|^{2}  \label{split_bregman_subproblem1}\\
        f_{i}^{k+1} = f_{i}^{k} + f_{i} - \Phi_{i} u^{k} \label{split_bregman_subproblem2}
    \end{numcases}
\end{subequations}
where $f_{i}^{0} = f_i$. It has been proved in \cite{Bregman_Iter} that the above iteration converges to the solution of (\ref{concept_opt1}). Moreover, as shown in \cite{XZhang1,XZhang2}, Bregman iteration is actually equivalent to alternatively decreasing the primal variable and increasing the dual variable of the Lagrangian of the problem (\ref{concept_opt1}).

In subproblem (\ref{split_bregman_subproblem1}), the data-fidelity term $E(u)$ co-occurs with the regularization term $J(u)$, which leads to a sophisticated solution. However, this subproblem can be solved efficiently using the forward-backward operator splitting method \cite{Lions,PL,GB}:
\begin{subequations}
    \begin{numcases}{}
        v^{k+1} = u^{k} - \delta \sum_{i=1}^{N} \Phi_{i}^{T} (\Phi_{i} u^{k} - f_{i}^{k}) \label{FB_step1} \\
        u^{k+1} = \argmin_{u} \bigg( \mu J(u) + \frac{1}{2\delta} \| u - v^{k+1} \|^2 \bigg) \label{FB_step2}
    \end{numcases}
\end{subequations}
where $\Phi_{i}^{T}$ denotes the adjoint operator of $\Phi_{i}$. The parameter $\delta$ is set to $1$ in experiments. The advantage of the above method is the separation of $\Phi_{i}$ and $J(u)$. The first step is called the forward step, which is actually the gradient descent of the data-fidelity term $E(u)$, and the second step is called the backward step, which can be solved efficiently for many choices of regularizer $J(u)$. For example, if we choose the total variation as the regularizer, then the subproblem is actually a standard Rudin-Osher-Fatemi (ROF) model \cite{rof}, which can be efficiently solved via the graph-cut \cite{graph_cut} or the split Bregman method \cite{split_bregman}. In our method, $J(u)$ is a mixed regularizer that contains non-local total variation (TV) and TV, so we refer to the subproblem (\ref{FB_step2}) as the mixed-ROF model.

\subsection{Finding an Efficient Solution for the Mixed-ROF Model}\label{fast_bregman}
No matter which ROF-like model we choose, current approaches always involve solving PDEs in each iteration, including the split Bregman method \cite{XZhang1,split_bregman}. Here, we extend the method proposed in \cite{fast_bregman}, which only handles the standard ROF model, to the general mixed ROF case, and provide a very simple but highly efficient and effective optimization algorithm without using PDEs. By replacing $J(u)$ with $J_s(u)$ and $J_t(u)$, the subproblem (\ref{FB_step2}) can be expanded as follows:
\begin{equation}\label{split_bregman_rof}
    \min_{u} \mu_{1} |\nabla_{w} u|_{1} + \mu_{2} \Big(|\nabla_{x} (u-u_{p})|_{1} + |\nabla_{y} (u-u_{p})|_{1}\Big) + \frac{1}{2} \|u - v\|_{2}^{2}
\end{equation}
The following new iteration schema can find the unique solution $u^{\ast}$ efficiently for the minimization problem (\ref{split_bregman_rof}).

Let $b_{w}^{0} = 0, b_{x}^{0} = 0, b_{y}^{0} = 0$ and $u^{1} = v$, for $k=1,2,\ldots$, the iteration is as follows:
\begin{align}\label{}
   &b_{w}^{k} = cut(\nabla_{w} u^{k} + b_{w}^{k-1}, \frac{\mu_1}{\lambda_1})\label{new_step1_in}\\
   &b_{x}^{k} = cut(\nabla_{x} u^{k} - \nabla_{x} u_p + b_{x}^{k-1}, \frac{\mu_2}{\lambda_2})\label{new_step2_in}\\
   &b_{y}^{k} = cut(\nabla_{y} u^{k} - \nabla_{y} u_p + b_{y}^{k-1}, \frac{\mu_2}{\lambda_2})\label{new_step3_in}\\
   &u^{k+1} = v + \frac{\lambda_1}{\mu_1} div_{w} b_{w}^{k} - \frac{\lambda_2}{\mu_2} (\nabla_{x}^{T} b_{x}^{k} + \nabla_{y}^{T} b_{y}^{k}) \label{new_step4_in}
\end{align}
where $shrink(x, \gamma) = x - \frac{x}{|x|}\ast \max(|x| - \gamma, 0)$.

In the above equations, $\nabla_{x}$ and $\nabla_{y}$ are the difference operators along the $x$ and $y$ directions, respectively, $\nabla_{w}$ denotes the non-local gradient operator, and $\text{div}_{w}$ is the non-local graph divergence operator (see the detailed mathematical definitions of all these operators in the Appendix). We prove the convergence of the proposed fast algorithm in the following theorem:
\begin{theorem}
\label{main_theorem}
For $k=1,2,\ldots$, let $b_{w}^{k}, b_{x}^{k},b_{y}^{k}$ and $u^{k+1}$ be given by the iteration (\ref{new_step1_in}) to (\ref{new_step4_in}). If $0< 20\lambda_1 + 4\lambda_2 <1$, then $\lim_{k\rightarrow \infty} u^{k} = u^{\ast}$.
\end{theorem}
\begin{proof}
    See the proof in Section \ref{proof-B}.
\end{proof}

\section{Distortion-Driven Image Fusion}\label{image_recon}

\begin{figure*}[htbp]
\setlength{\abovecaptionskip}{0pt}  %±êÌâµÄÉϱ߽ç
\setlength{\belowcaptionskip}{0pt} %±êÌâµÄϱ߽ç
\renewcommand{\figurename}{\footnotesize{Figure}}
\centering
\includegraphics[width=0.8\textwidth]{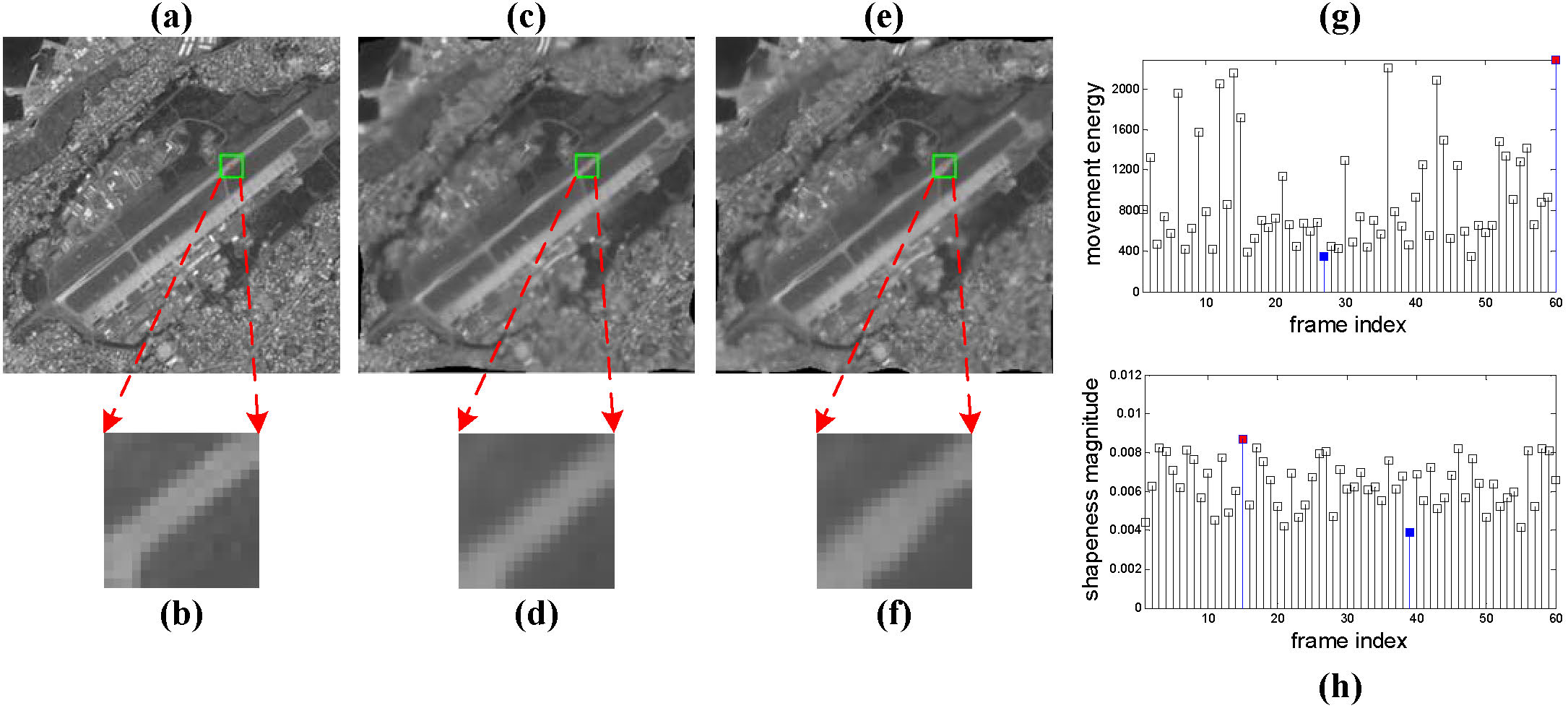}
\caption{\footnotesize{Near-stationary patch vs. near-diffraction-limited patch in the airport sequence. (a) Ground truth; (b) Zoomed ground truth patch from (a); (c) Frame contains near-stationary patch; (d) Zoomed near-stationary patch from (c); (e) Frame contains near-diffraction-limited patch; (f) Zoomed near-diffraction-limited patch from (e); (g) Movement energy of the local patch over 60 frames; (h) Intensity variance of local patch over 60 frames.}}
\label{fig:stationary-vs-diffraction}
\end{figure*}

%ÕâÒ»¶Î¿ªÍ·ÐèҪͨ¹ýÒ»¸öʵÑé˵Ã÷½ö½öʹÓÃsharpnessµÄÅбð±ê×¼ÊDz»¹»µÄ£¬»¹ÐèÒªmovement energy½øÐи¨ÖúÅжÏ

\begin{algorithm}
 \DontPrintSemicolon
    Given a registered sequence $\{R_k\}$ and the push-forward motion field sequence $\{ \overleftarrow{F_{k}}\}$ by inversing $\{ \overrightarrow{F_{k}}\}$, divide each frame into $L\times L$ overlapping patches centered at each pixel, and calculate the intensity variance and the movement energy of each patch as a local stationary measure.\\

    For a patch sequence $\{r_k\}$ centered at location $\mathbf{x}$, detect the most stationary one $r_{k^{*}}$ by local variance and movement energy.\\

    Set $r_{k^{*}}$ as a reference patch, and restore its center pixel value using spatial-temporal kernel regression. Assign this value to the pixel $Z[\mathbf{x}]$.\\

    Go to the next pixel and return to step 2.
 \caption{Procedure for reconstructing a stationary image from registered frames}\label{algorithm:sir}
\end{algorithm}

In this section, by introducing the concept of the \textbf{near-stationary patch}, we fuse the registered sequence $\{R_k\}$ into a single image $Z$, which can be deblurred using a space-invariant deconvolution method. The fusion steps are summarized in Algorithm \ref{algorithm:sir}. Different from our approach, \cite{xzhu2} builds the image $Z$ by fusing the diffraction-limited patches, which can be detected using local sharpness measures. However, due to the registration error, some detected diffraction-limited patches still contain local distortion. Examples are shown in Fig. \ref{fig:stationary-vs-diffraction} where the patches (d) and (f) are obtained by near-stationary detection and near-diffraction-limited detection, respectively. Compared with (f), the near-stationary patch (d) shows less local deformation and more closely resembles the latent true patch (b). Therefore, employing a local sharpness measure alone is insufficient to recover the image details. In the following subsections, we will first describe how to detect near-stationary patches, and then detail the proposed distortion-driven spatial-temporal kernel regression used for reducing the noise caused by sensor and registration errors during the fusion process.

\subsection{Near-Stationary Patch Detection}\label{nspd}
In turbulence free or near turbulence free conditions, the light reflected from the scene through the atmosphere is perpendicular to the camera. Therefore, any scene inside local regions should be clearly observed by the camera without distortion or blur. Such local regions are referred to as stationary patches. Suppose the $k^{\ast}$th patch is found to be near-stationary: we propose to detect this near-stationary patch by taking both local sharpness and local movement energy into account. According to \cite{xzhu2}, the local sharpness of a $L\times L$ patch is determined by its variance in intensity, which is
\begin{equation}\label{local_sharpness}
s_{\mathbf{x}}^{k} = \frac{1}{L^{2}-1}\sum_{\mathbf{x}} (r_k[\mathbf{x}] - \bar{r}_k)^{2}
\end{equation}
where $\bar{r}_k$ represents the mean value of patch $r_k$. For further details on the $s_{\mathbf{x}}^{k}$ calculation, please refer to \cite{xzhu2}.

The local movement energy of an $L\times L$ patch $r_{k}$ can be calculated as:
\begin{equation}\label{movement_energy}
e_{\mathbf{x}}^k = \sum_{\mathbf{x}\in \Omega(r_k)}\mid \overleftarrow{f_{k}}[\mathbf{x}] \mid^{2}
\end{equation}
where $\overleftarrow{f_{k}}[\mathbf{x}]$ denotes the deformation vector at position $\mathbf{x}$ from the $k$-th deformation field in sequence $\{\overleftarrow{F_i}\}$ (the deformation field $\overrightarrow{F_i}$, which warps the distorted frame $f_i$ to the reference image, and its inverse field $\overleftarrow{F_i}$ are calculated from B-spline registration), and $\Omega(r_k)$ is the support region of the local patch $r_k$. Consequently, given the stationary measurement, we select the energy patch $r_{k^{\ast}}$ with the lowest movement as the near-stationary patch from the ten sharpest ones in the patch sequence $\{r_k\}$. Then, $r_{k^{\ast}}$ is used as a reference patch to restore its center pixel value; this is described in the next subsection.

\subsection{Fusion by Spatial-Temporal Kernel Regression}
This subsection describes how one single image $Z$ is generated by fusing the detected near-stationary patches. To avoid possible artifacts that may appear during subsequent deblurring, noise in the selected near-stationary patches needs to be suppressed. Zhu and Milanfar \cite{xzhu2} employed zero-order kernel regression to estimate the value of a pixel at $\mathbf{x}$ in its reference frame $\delta$ by
\begin{equation}\label{eqn:kernel_regression}
    \widehat{q}_\delta[\mathbf{x}] = \frac{\sum_{k} U(\mathbf{x};k,\delta)r_k[\mathbf{x}]}{\sum_{k} U(\mathbf{x};k,\delta)}
\end{equation}
where $U(\cdot)$ denotes a patch-wise photometric distance which can be calculated by a Gaussian kernel function
\begin{equation}\label{}
    U(\mathbf{x};k,\delta) = \exp\bigg(\frac{2\sigma_n^2}{\mu^2} - \frac{\parallel r_k - r_\delta \parallel^2}{L^2 \mu^2}\bigg),
\end{equation}
where $\sigma_{n}^2$ is the noise variance (we set $\sigma_{n}^2 = 2$ in our experiments), $L^2$ denotes the total number of pixels in the patch, and the scalar $\mu$ is the smoothing parameter \cite{kernelregression1}.

However, this denoising method only takes temporal information into account and ignores the registration error in the image domain. From Fig. \ref{fig:stationary-vs-diffraction} we can see that the more the movement energy of a patch, the less likely it is to have its deformation thoroughly removed. Therefore, $r_k[\mathbf{x}]$ in Eq. (\ref{eqn:kernel_regression}) may not represent the true value of the pixel $\mathbf{x}$. Hence, to reconstruct a single image $Z$, a spatial-temporal kernel regression is used to restore each pixel. In the spatial domain (image domain), distortion-driven asymmetric steering kernel regression is used to remove the deviation of each central pixel of the patch of $r_k$. In the temporal domain, patch-wise temporal regression is used to correct the central pixel and further reduce the noise level \cite{xzhu2}. The details of the spatial-temporal kernel regression are described in the following two subsections.

\subsubsection{Point-Wise Spatial Asymmetric Steering Kernel Regression}
\begin{figure*}[htbp]
\setlength{\abovecaptionskip}{0pt}  %±êÌâµÄÉϱ߽ç
\setlength{\belowcaptionskip}{0pt} %±êÌâµÄϱ߽ç
\renewcommand{\figurename}{\footnotesize{Figure}}
\centering
\includegraphics[width=0.6\textwidth]{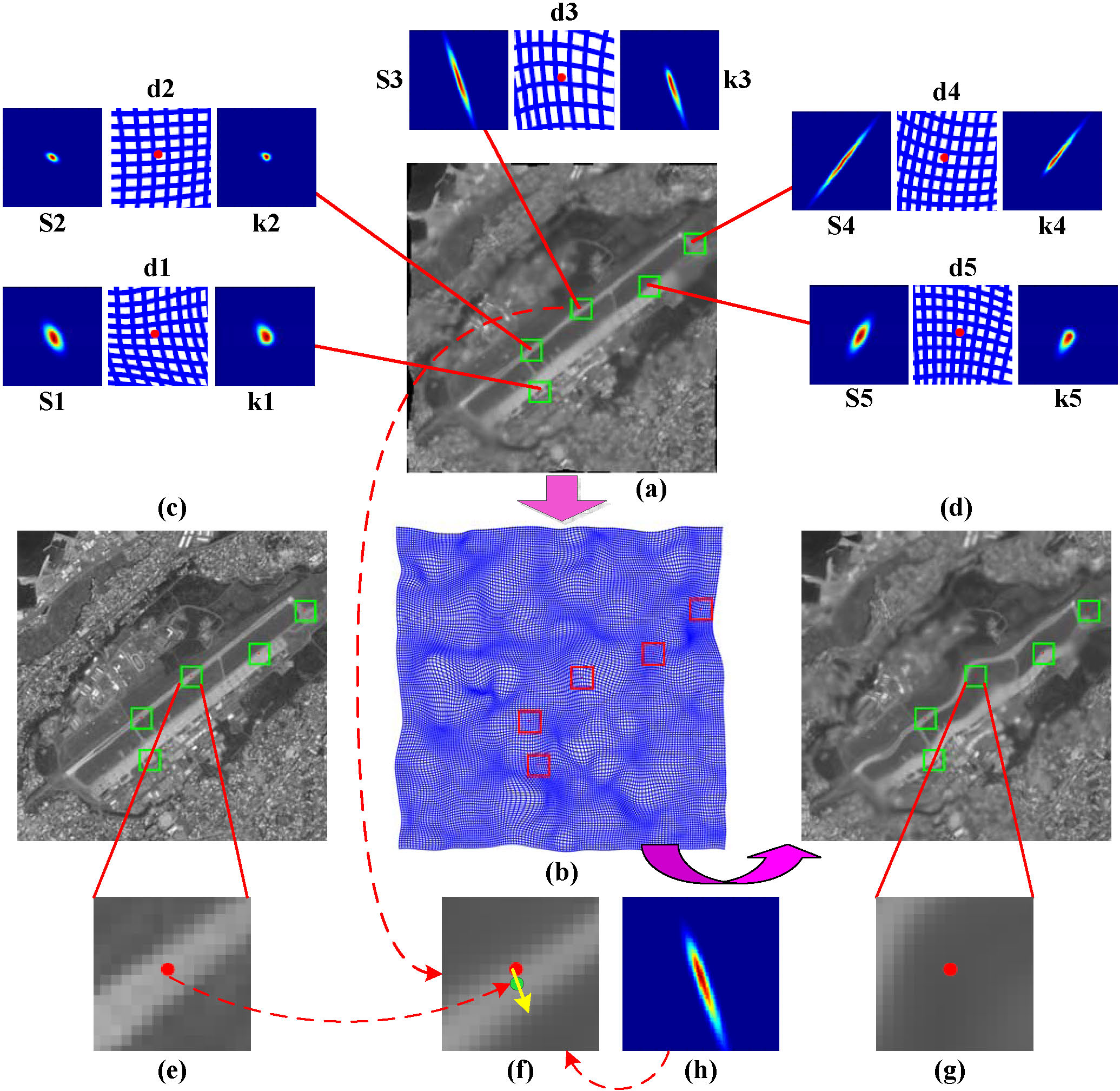}
\caption{\footnotesize{Illustration of distortion-driven asymmetric steering kernel.}}
\label{fig:show_adskr_kernel}
\end{figure*}

Directly finding a true pixel value is a nontrivial problem. We resort to calculating $\widetilde{r}_k[\mathbf{x}]$, which denotes the corrected value of a pixel at position $\mathbf{x}$ in the registered frame $R_k$ to replace $r_k[\mathbf{x}]$ in Eq. (\ref{eqn:kernel_regression}). Then, $\widetilde{r}_k[\mathbf{x}]$ can be estimated using the zero-order kernel regression (Nadaraya-Watson estimator \cite{zero-regression})
\begin{equation}\label{regression_model}
    \widetilde{r}_k[\mathbf{x}] = \frac{\sum\limits_{x_j\in r_k} K_{dskr}(\mathbf{x}_j - \mathbf{x}) r_{k}[\mathbf{x}]}{\sum\limits_{x_j\in r_k} K_{dskr}(\mathbf{x}_j - \mathbf{x})}
\end{equation}
where $K_{dskr}(\cdot)$ represents a distortion-driven kernel function that depends on the local deformation field. Such a kernel can indicate a pixel's direction of deviation; in other words, the weights should be larger in the deviated direction (see Fig. \ref{fig:show_adskr_kernel} $s1$ to $s5$). Inspired by the steering kernel \cite{kernelregression2}, we propose a two-step approach to construct a distortion-driven kernel function.

First, measure the dominant orientation of the local deformation field and build the symmetric kernel function. The orientation is the singular vector corresponding to the largest (nonzero) singular value of the local deformation matrix
\begin{equation}\label{formular:dominant_direction}
M_k = \begin{bmatrix}
\vdots & \vdots\\
\overleftarrow{f_{x}}(\mathbf{x}_{j}) & \overleftarrow{f_{y}}(\mathbf{x}_{j})\\
\vdots & \vdots
\end{bmatrix}
= U_{k}S_{k}V_{k}^{T},~~~\mathbf{x}_{j}\in \Omega(r_{k}),
\end{equation}
where $\overleftarrow{f_{x}}$ and $\overleftarrow{f_{y}}$ are the displacement vectors along the $x$ and $y$ direction, and $\Omega(r_{k})$ is the local support region of $r_k$. $U_{k}S_{k}V_{k}^{T}$ is the truncated singular value decomposition of $M_{k}$, and $S_{k}$ is a diagonal $2\times 2$ matrix representing the energy in the dominant direction. Then, in the first column of the orthogonal matrix $V_k$, $v_{1}=[\nu_{1},\nu_{2}]^{T}$ defines the dominant orientation angle $\theta_{k}= \arctan(\frac{\nu_{2}}{\nu_{1}})$. This means that the singular vector corresponding to the largest singular value of $M_k$ represents the dominant orientation of the local deformation field. The elongation parameter $\sigma_i$ and scaling parameter $\gamma_i$ are calculated by
\begin{equation}\label{formular:elongation1}
    \sigma_i = \frac{s_1 + \lambda^{'}}{s_2 + \lambda^{'}}, \quad \lambda^{'}\geq 0, \quad
    \gamma_i = \bigg(\frac{s_1 s_2 + \lambda^{''}}{M}\bigg)^{\frac{1}{2}},
\end{equation}
where $s_1$ and $s_2$ are the diagonal elements of the matrix $S_{k}$, and $\lambda^{'}$ and $\lambda^{''}$ denote the regularization parameters. Given the parameters mentioned above, we can calculate the steering kernel function as in \cite{kernelregression2} (hereafter, the symmetric steering kernel is referred to as $K^{s}_{dskr}(\cdot)$). Fig. \ref{fig:show_adskr_kernel} is a visual illustration of the steering kernel footprints for different local deformation regions. Image (a) is one registered frame from $\{R_i\}$, and the patches $d1$ to $d5$ (cropped from deformation field (b)) are the local deformation fields of the pixels marked by red dots in image (a). As shown in Fig. \ref{fig:show_adskr_kernel}, the size and shape of each symmetric steering kernel ($s1$ to $s5$) are locally adapted to the corresponding deformation field. For example, $s2$ is blunt due to the small movement energy of its corresponding deformation field, while $s3$ is sharp because the movement energy of $d3$ is much larger than that of $d2$.

Second, construct the asymmetric steering kernel using the dominant orientation. State-of-the-art non-rigid registration approaches are usually unable to correct a local deformation when the movement energy is very large. Nevertheless, due to the smoothness of the B-spline-based registration, the true pixel value $\widetilde{r}_k[\mathbf{x}]$ often appears in the opposite direction of the dominant orientation. As shown in Fig. \ref{fig:show_adskr_kernel}, patches (g) and (f) are cropped from one observed frame (d) and its corresponding registered frame (a), respectively, and patch (e) is cropped from the ground truth image (c) (the central pixels of (e), (f), and (g) are marked by red dots). Comparing patch (e) with patch (f), the two central pixels represent different positions due to the registration error. In fact, the central pixel of patch (e) corresponds to the pixel marked in green in patch (f). Consequently, to estimate $\widetilde{r}_k[\mathbf{x}]$ more precisely, we need to further shrink the footprint of the local kernel; the weights along the inverse dominant orientation should be larger than those along the dominant orientation. Hence, we construct a local asymmetric kernel function $K^{a}_{dskr}(\cdot)$ by utilizing the {\it Asymmetric Gaussian} proposed in \cite{asymmetric-kernel}, where the dimension needed to be asymmetric can be calculated by
\begin{equation}\label{nbp}
\begin{aligned}
\textit{A}(z_i; \mu^{z}_i, \rho^{2}_i, r_i) = \frac{2}{\sqrt{2\pi}} \frac{1}{\sqrt{\rho^{2}_i}(r_i + 1)}
\left\{
   \begin{aligned}
   &exp\Big( -\frac{(z_i - \mu^{z}_i)^{2}}{2\rho^{2}_i} \Big) \qquad \text{if} \quad z_i >  \mu^{z}_i\\
   &exp\Big( -\frac{(z_i - \mu^{z}_i)^{2}}{2r^{2}_i \rho^{2}_i} \Big) \qquad \text{otherwise} \\
   \end{aligned}
\right.
\end{aligned}
\end{equation}
where $z_i$ is along the dominant orientation of the symmetric kernel $K^{s}_{dskr}(\cdot)$, $\mu^{z}_i$ and $\rho_i^{2}$ are the mean and variance corresponding to $z_i$, and $r_i$ denotes the asymmetric coefficient with $r_i = 1$ being equivalent to the ordinary Gaussian. In our implementation, $r_i$ is usually set to $0.5\sqrt{\sigma_i}$, where $\sigma_i$ is the elongation coefficient defined in Eq. (\ref{formular:elongation1}). An example of an asymmetric kernel is shown in Fig. \ref{fig:show_adskr_kernel} $k1$ to $k5$. Due to the computational complexity, we only use spatial kernel regression for the pixels whose movement energy is greater than a pre-defined threshold. For the remaining pixels with relatively low movement energy, we directly assign $r_k[\mathbf{x}]$ to $\widetilde{r}_k[\mathbf{x}]$.

\subsubsection{Patch-Wise Temporal Kernel Regression}
Once the corrected pixel value $\widetilde{r_k}[\mathbf{x}]$ in each frame has been calculated, we can further estimate the pixel value at position $\mathbf{x}$ by means of spatial-temporal kernel regression for a single image $Z$
\begin{equation}\label{spatial-temporal-regression}
\begin{aligned}
    \widehat{q}_\delta[\mathbf{x}] = \frac{\sum_{k} U(\mathbf{x};k,\delta)\widetilde{r_k}[\mathbf{x}]}{\sum_{k} U(\mathbf{x};k,\delta)}
    = \frac{\sum_{k} U(\mathbf{x};k,\delta)\frac{\sum\limits_{x_j\in r_k} K^{a}_{dskr}(\mathbf{x}_j - \mathbf{x}) r_{k}[\mathbf{x}]}{\sum\limits_{x_j\in r_k} K^{a}_{dskr}(\mathbf{x}_j - \mathbf{x})}}{\sum_{k} U(\mathbf{x};k,\delta)}.
\end{aligned}
\end{equation}
Spatial-temporal kernel regression is illustrated in Fig. \ref{fig:spatial_temporal_regression}: the yellow box denotes the spatial regression that exploits the local deformation structure to restore a certain pixel only in the image domain, and the red line represents the temporal regression, which fuses the information across all the corrected pixels $\{\widetilde{r_k}[\mathbf{x}]\}_{k=1}^{N}$ in the temporal domain. Due to utilization of both spatial and temporal information, we refer to such regression as the \textbf{Spatial-Temporal Kernel Regression}.

\begin{figure}[htbp]
\setlength{\abovecaptionskip}{0pt}  %±êÌâµÄÉϱ߽ç
\setlength{\belowcaptionskip}{0pt} %±êÌâµÄϱ߽ç
\renewcommand{\figurename}{Figure}
\centering
\includegraphics[width=0.5\textwidth]{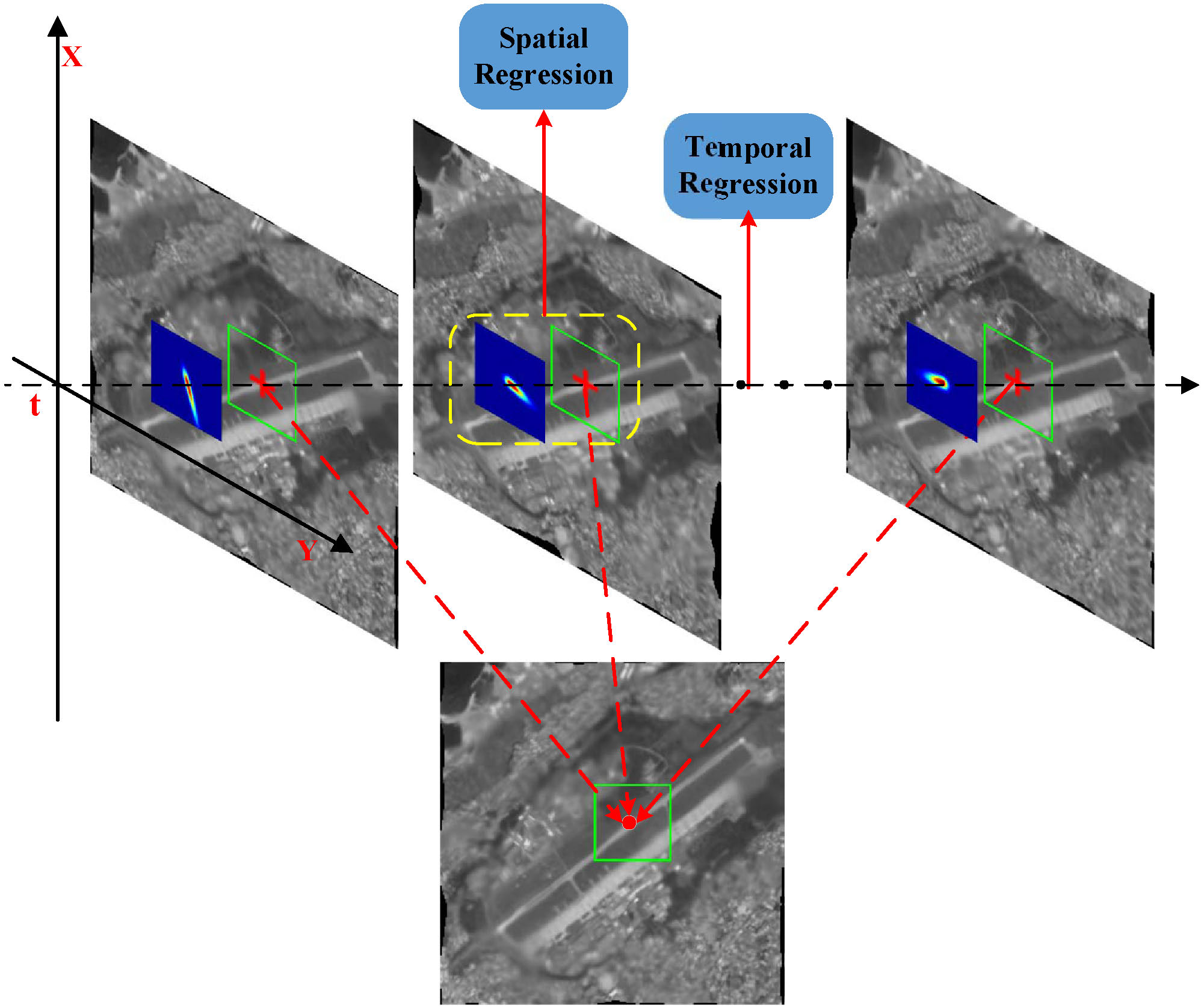}
\caption{Illustration of distortion-driven spatial-temporal kernel regression.}
\label{fig:spatial_temporal_regression}
\end{figure}

\section{Space-invariant Deconvolution}\label{final-deconvolution}		% Modify as required for article
A post-process is needed to correct the diffraction-limited blur that still exists in $Z$. We exploit the blind deconvolution algorithm \cite{deblur-1} to calculate a final output (corresponding to step $18$ in Algorithm \ref{propose_alg}). In the degradation model
\begin{equation}\label{blur}
    Z = L\otimes h + \varepsilon
\end{equation}
the deblurring procedure can be described using the following optimization problem
\begin{equation}\label{deblur-optimization}
    \argmin_{F,h}\|Z - L\otimes h\|_{2}^{2} + \gamma_1 R_{f}(L) + \gamma_2 R_{h}(h)
\end{equation}
where $\varepsilon$ denotes the error caused by the fusion process in Section \ref{image_recon}, and $R_f$ and $R_h$ are the regularizers used to restrict the latent sharp image $L$ and the blur kernel $h$ based on their own prior knowledge. As suggested in \cite{deblur-1}, the regularization term for $L$ is defined as
\begin{equation}\label{L-term}
    R_f(L) = \|\rho(\partial_x L) + \rho(\partial_y L)\|_1
\end{equation}
where $\partial_x L$ and $\partial_y L$ represent the derivatives of $L$ in horizontal and vertical directions, respectively. The $\rho(\cdot)$ is defined as
\begin{equation}\label{h-term}
\rho(x)
\left\{
   \begin{aligned}
   &-\theta_1|x| \qquad &x\leq l_t\\
   &-(\theta_2 x^2 + \theta_3) \qquad &x>l_t \\
   \end{aligned}
\right.
\end{equation}
where $l_t$, $\theta_1$, $\theta_2$, and $\theta_3$ are all fixed parameters. The sparse constraint is also imposed on the PSF $h$: $R_h(h) = \|h\|_1$. The detail of the algorithm used for solving the optimization problem (\ref{deblur-optimization}) can be found in \cite{deblur-1}. In our experiments, two types of parameter setting are used for the simulated data and the real data, respectively. In simulated experiments, ``kernelWidth'', ``kernelHeight'', ``noiseStr'', and ``deblurStrength'' are set to $5$, $5$, $0.03$, and $0.2$, respectively; all other parameters use the default settings, and the description of all the parameters can be found in \cite{deblur-1}. For real sequences, setting the same parameters to $(7, 7, 0.03, 0.5)$ can reproduce the displayed results. Finally, we summarize the proposed atmospheric turbulence removal algorithm in Algorithm \ref{propose_alg}.

\SetAlFnt{\footnotesize}{
\begin{algorithm}[]\label{propose_alg}
\SetAlgoLined
\caption{Atmospheric Turbulence Removal Algorithm}
\KwIn{Distorted image set $V\in \mathbb{R}^{w\times h \times n} = \{f_1,\ldots, f_n\}$ \\}
\KwOut{Restored image $L$\\}
\BlankLine
$u \leftarrow \text{Low\_Rank\_Decomposition}(V)$\;
\tcp*[h]{out\_loop}\\
\While{$iter < iter\_max$ }
{
    \tcp*[h]{middle\_loop}\\
    \While{$\sum_{i} \| \Phi_{i}u - f_{i} \| > \varepsilon_{1}$}
    {
        Initialize the $\tilde{f}_{i} = f_{i}$\;
        Compute each $\Phi_i$ which warps $u$ to the observed image $f_i$ using B-spline registration\;
        \tcp*[h]{inner\_loop}\\
        \While{$\sum_{i} \| \Phi_{i}u - \tilde{f}_{i} \| > \varepsilon_{2}$}
        {
            \BlankLine
            $v = u - \delta \sum_{i} \Phi_{i}^{T} (\Phi_i u - \tilde{f}_{i})$ \;
            Solve $u = \argmin \nolimits_{u} \mu J(u) + \frac{1}{2\delta} \| u - v \|^2 $ by iteration in Section 3.2\;
        }
        $\tilde{f}_{i} \leftarrow \tilde{f}_{i} + f_{i} - \Phi_{i}u$\;
    }
    $\hat{u} = u$\;
    %Compute new $\Phi_{i}^{'}$ which warps $\hat{u}$ to $f_{i}$ using B-spline registration\;
    $\{R_i\} \leftarrow$ register each frame $f_i$ to $\hat{u}$\;
    $V \leftarrow \{R_i\}$\;
    $u \leftarrow \text{Low\_Rank\_Decomposition}(V)$\;
}

Construct single frame $Z$ from $\{R_i\}$ using Spatial-Temporal kernel regression in Section 4\;
Deblur the image $Z$ to get the latent sharp image $L$ in Section 5\;
\textbf{Return} $L$\;
\end{algorithm}}

\section{Experimental Results and Analysis}\label{experiment}
%×ÜÌå˵һÏÂʵÑé·½·¨ºÍ²½Ö裬ÀýÈç·ÂÕæʵÑéºÍÕæʵÊý¾ÝÉϵÄʵÑé
This section presents extensive experimental validation of the proposed restoration method. We first show that low-rank decomposition improves the quality of the reference image compared to traditional methods. Next, we compare the results optimized using the proposed variational model to those using a spatial regularizer alone, to illustrate the advantages of our model. Finally, both qualitative and quantitative methods are used to evaluate the performance of the proposed method in comparison with several state-of-the-art methods. For quantitative evaluation, the {\it Peak Signal to Noise Ratio} (PSNR) and {\it Structural Similarity Index} (SSIM) are adopted to objectively evaluate the quality of the restored images.

%----------»¹ÐèÒªÌáµ½½øÐжԱÈʵÑéµÄ·½·¨£¬6ÖÖ--------------------
For all the experiments, the intervals of the control points in the registration are set to $\varepsilon_x = \varepsilon_y = 16$ pixels, and the patch size $L$ of $r_k$ is set to $13$. For non-local total variance, the local patch size is set to $5\times 5$ (the support region for a certain pixel), the search window size is set to $21\times 21$ (the region for searching for similar patches), and the number of best neighbors is set to $10$ (the number of accepted similar pixels in the search window). Moreover, for the number of optimized iterations (Algorithm \ref{propose_alg} in Section \ref{final-deconvolution}), $out\_loop = 1$ is sufficient for all the test sequences and most of degradation cases, and the $middle\_loop$ and $inner\_loop$ are set to $3$ and $10$, respectively. Furthermore, in the variational model, the parameters $\lambda_1$ and $\lambda_2$ should satisfy the condition defined in Lemma \ref{lemma1} in Section \ref{proof-B}, so they are chosen to be equal: $\lambda_1 = \lambda_2 = 0.02$. With respect to the trade-off between spatial and temporal regularization, we choose $\mu_1 = 0.5$ and $\mu_2 = 0.25$. The proposed method is implemented in Matlab with MEX, and all the experiments are performed on a standard Intel Core i7 2.8GHz computer. The code and data of the proposed method are available at \textbf{https://sites.google.com/site/yuanxiehomepage/}.

We compare the proposed method with five representative algorithms: the lucky region method \cite{luckyregion3} (\textbf{Lucky region}), principal components analysis for atmospheric turbulence \cite{atomspheric-pca} (\textbf{PCA}), the data-driven two-stage approach for image restoration \cite{twostage} (\textbf{Twostage}), Bregman iteration and non-local total variance for atmospheric turbulence stabilization \cite{bregman_distort} (\textbf{BNLTV}), and near-diffraction-limited-based image reconstruction for removing turbulence \cite{xzhu2} (\textbf{NDL}). It is worth noting that the Twostage method was originally designed for recovering the true image of an underwater scene from a sequence distorted by water waves. Even though the medium is different, Twostage still achieves reasonable results on images distorted by air turbulence, and therefore the comparison and inclusion of this algorithm is valid. The respective authors provide the Twostage \cite{twostage} and NDL \cite{xzhu2} code, and the parameters remain unchanged. We have implemented the code for the other algorithms, which remains faithful to the original papers.

\subsection{Quality of the Low-rank-based Reference Image}\label{experiment-low-rank}

\begin{figure*}[htbp]
\setlength{\abovecaptionskip}{0pt}  %±êÌâµÄÉϱ߽ç
\setlength{\belowcaptionskip}{0pt} %±êÌâµÄϱ߽ç
\renewcommand{\figurename}{Figure}
\centering
\includegraphics[width=\textwidth]{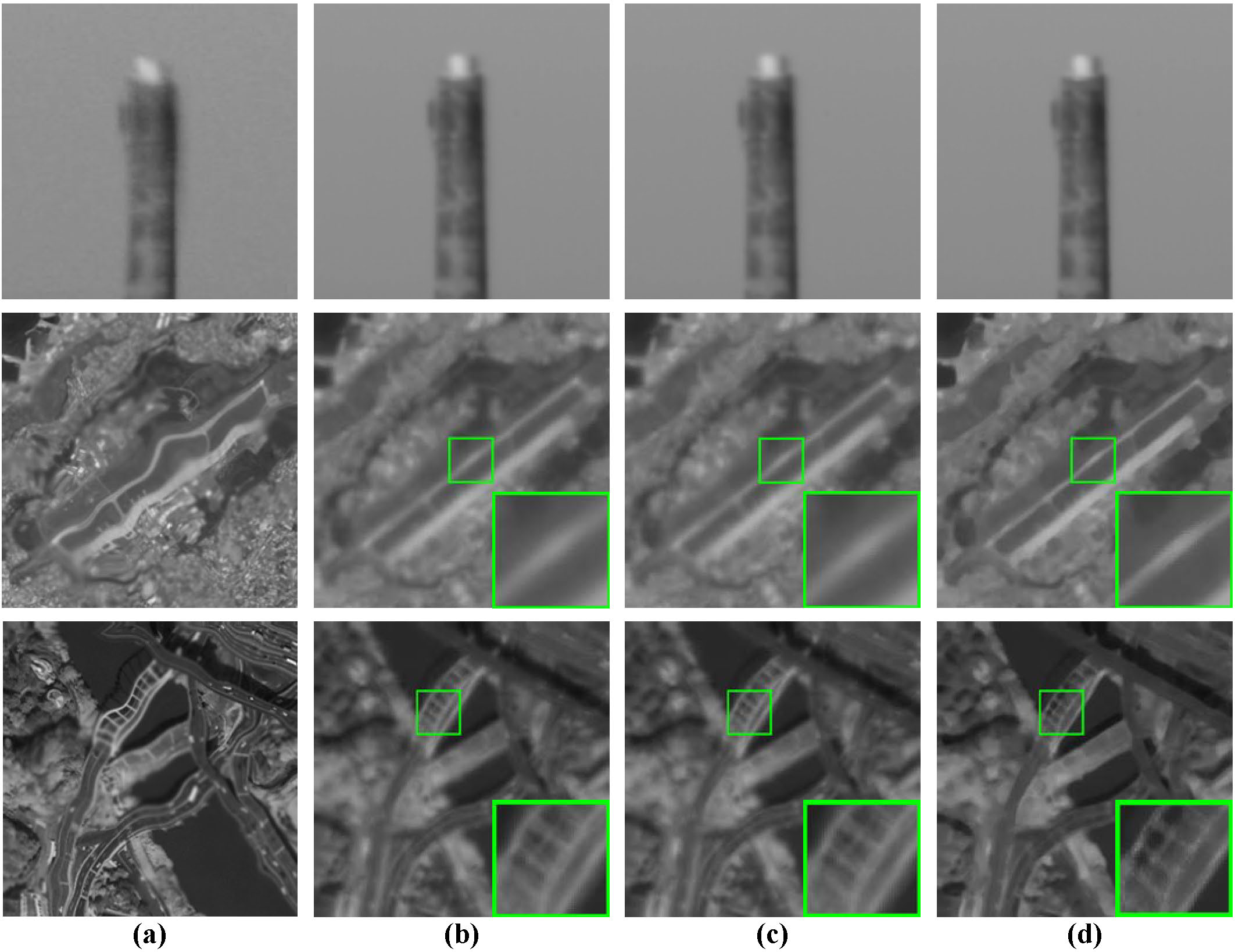}
\caption{Comparison of the reference images constructed by three methods}
\label{fig:low_rank_vs_temporal}
\end{figure*}
In this subsection we compare the visual quality of the reference images generated by our proposed method, temporal averaging, and the lucky region approach \cite{luckyregion3}. The reference images are presented in Fig. \ref{fig:low_rank_vs_temporal}, where the first column (a) contains three observed frames from the 'chimney', 'airport', and 'city\_strong' sequences; the other three columns are the results of temporal averaging (b), lucky-region (c), and low-rank decomposition (d), respectively. In the chimney sequence, the low-rank decomposed reference image is slightly sharper than the other two images. Nevertheless, due to the image being affected by only weak turbulence, the observed difference between the three reference images is small. However, in the case of severe turbulence (airport sequence), the effects are more noticeable. The magnified area of the region of interest (marked by the green box) can be seen in the lower right hand corner of each image. Low-rank decomposition produces sharp edges, while the images produced by temporal averaging and the lucky region methods remain blurred; similar results are seen in the city\_strong sequence. In addition, the other methods produce heavy edge artifacts, {\it e.g.}, in the region of the bridge, with much sharper and clearer edges produced by our method, even in the case of relatively indistinct edges. These indistinct edges can subsequently be enhanced by using the variational model.

%---------¸Ã²¿·ÖÐèÒª¼ÓÉÏʱ¼ä¸´ÔӶȵķÖÎö----------------
\subsection{Advantages of the Spatial-Temporal Regularization}

\begin{figure*}[htbp]
\setlength{\abovecaptionskip}{0pt}  %±êÌâµÄÉϱ߽ç
\setlength{\belowcaptionskip}{0pt} %±êÌâµÄϱ߽ç
\renewcommand{\figurename}{Figure}
\centering
\includegraphics[width=0.9\textwidth]{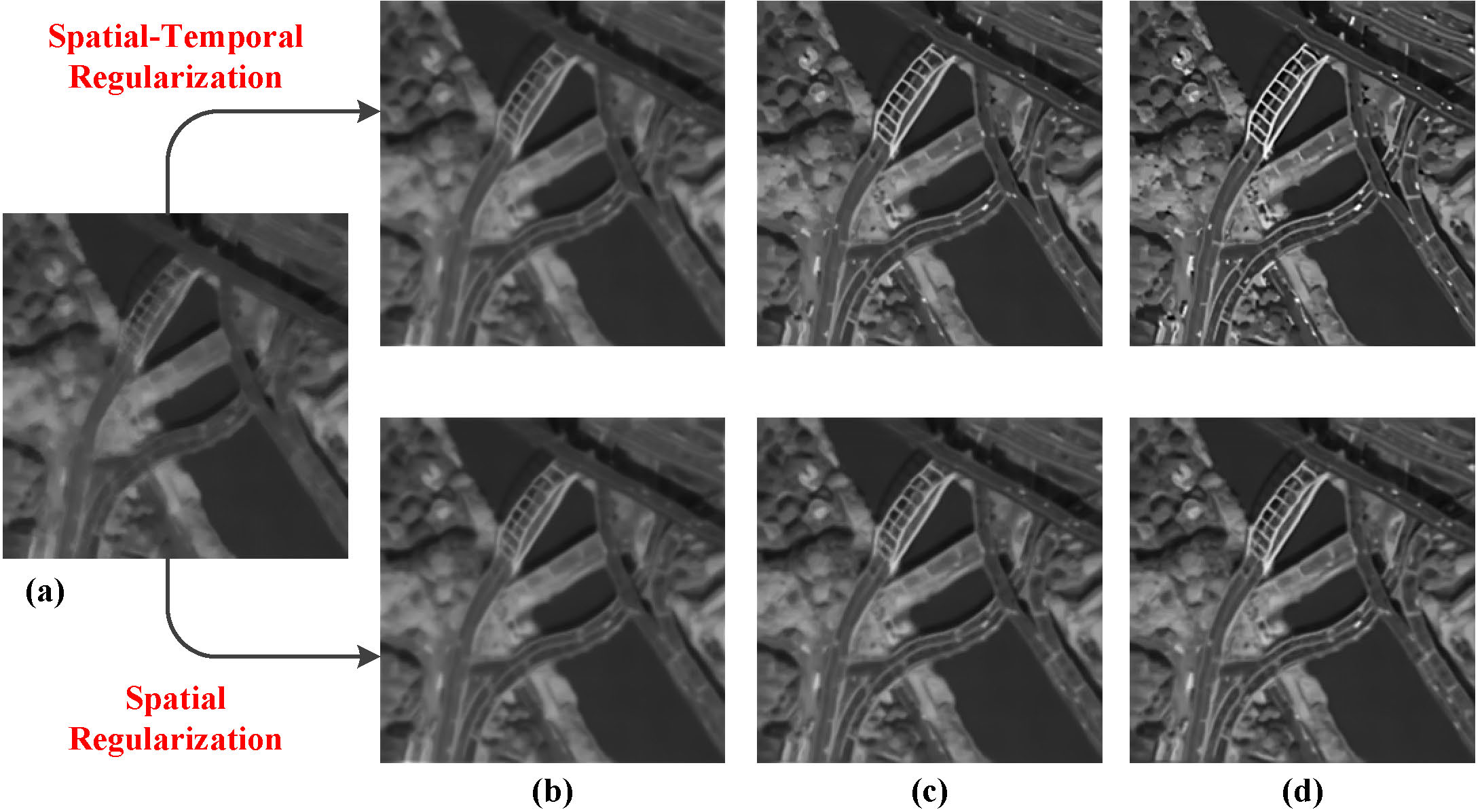}
\caption{Optimization of the two regularizers for the city\_strong sequence.}
\label{fig:city_strong_compare_sp_vs_s}
\end{figure*}
\begin{figure*}[htbp]
\setlength{\abovecaptionskip}{0pt}  %±êÌâµÄÉϱ߽ç
\setlength{\belowcaptionskip}{0pt} %±êÌâµÄϱ߽ç
\renewcommand{\figurename}{Figure}
\centering
\includegraphics[width=0.9\textwidth]{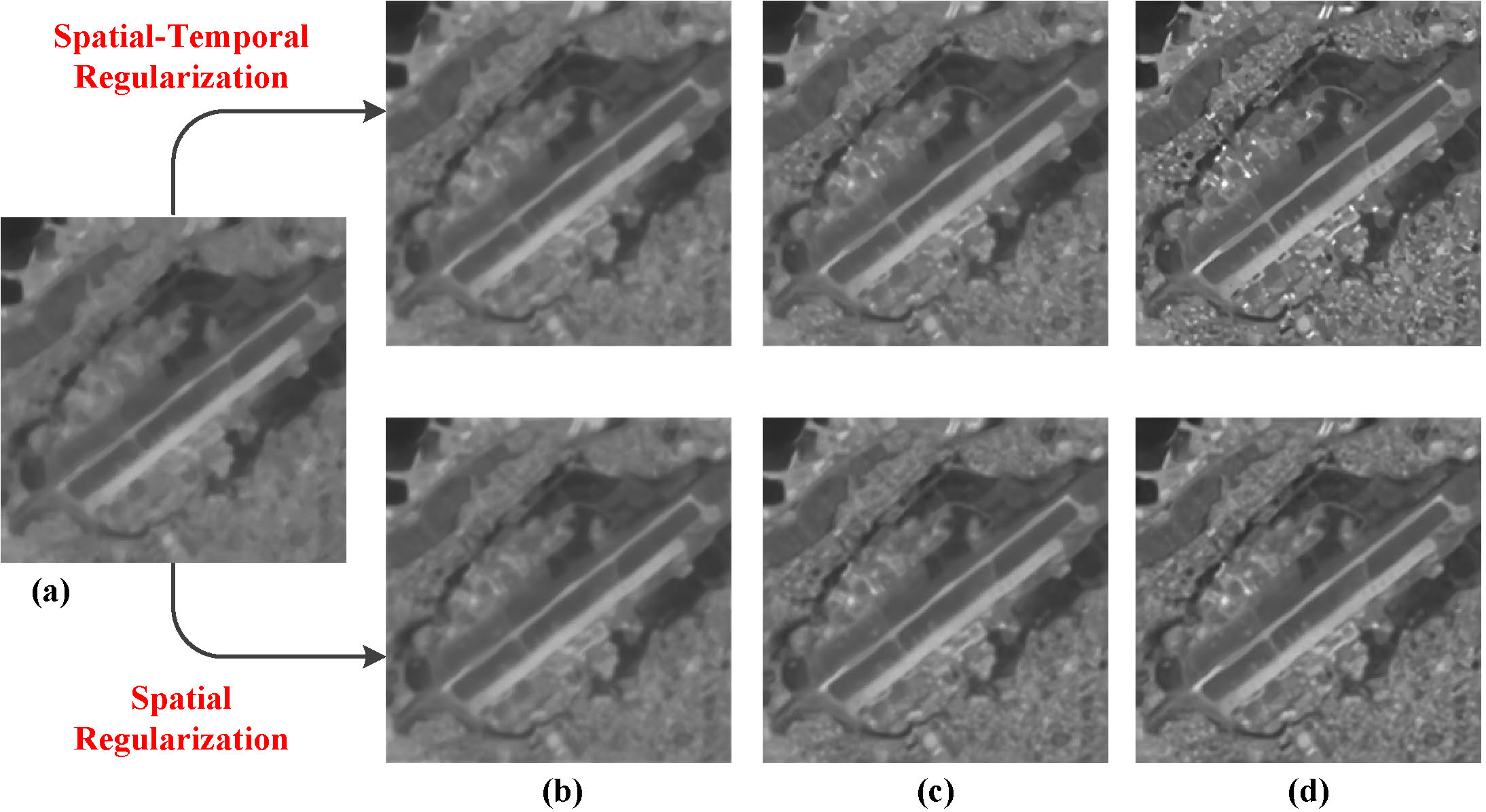}
\caption{Optimization of the two regularizers for the airport sequence.}
\label{fig:airport_compare_sp_vs_s}
\end{figure*}
This subsection illustrates the advantages of spatial-temporal regularization. We therefore compare the reference image $u$ optimized using the two regularization terms: (1) the spatial-temporal regularizer as proposed; and (2) the spatial regularizer (traditional NLTV). Figs. \ref{fig:city_strong_compare_sp_vs_s} and \ref{fig:airport_compare_sp_vs_s} show the process of optimization using the two regularizers on the city\_strong and airport sequences, respectively. Image (a) is the initial reference image obtained by low-rank decomposition, and the other three columns from left to right present the first, second, and the final optimization steps (the three steps correspond to the $middle\_loop$ in Algorithm \ref{propose_alg}).

As can be seen by comparing (a) with the bottom row of (d) in Fig. \ref{fig:city_strong_compare_sp_vs_s}, the spatial regularizer enhances 'strong' edges (we refer to object's profile as the strong edges, and consider the edges which are inside an object to describe local structures as the weak edges), but the weak edges are not optimally recovered and remain blurry. In contrast, the spatial-temporal regularizer sharpens both strong and weak edges (top row Fig. \ref{fig:city_strong_compare_sp_vs_s}). This phenomenon can be explained by the following two factors: 1) NLTV is limited by its preservation of texture and local structures in an image due to the smoothness caused by the weighted averaging of non-local self-similar patterns. Unfortunately, these local structures usually include weak edges, resulting in different degrees of recovery of strong and weak edges. 2) The total variation employed by $J_t(u)$ can force the energy difference of two optimized results to converge on the sparse discontinuities, on which the edges lie in the functional space. Therefore, $J_t(u)$ can enhance both strong and weak edges. In summary, combining the non-local TV-based spatial regularizer with the local TV-based temporal regularizer does noticeably improve the quality of the reference image. More illustrative results can be seen in Fig. \ref{fig:airport_compare_sp_vs_s}.

Additionally, we analyze the computational complexity and CPU time of the proposed fast algorithm and the split Bregman method (the iterative optimization algorithm is presented in Section \ref{SB} (from Eqs. (\ref{sb-opt1}) to (\ref{sb-opt4})). For a fixed number of iterations, both the split Bregman method and our approach are linear in $N$ (the number of pixels), since each step only contains addition and scalar multiplication operations. Therefore, we compare them in a different way by accounting for the number of atom operators in the key steps of each method, which in this case are Eqs. (\ref{new_step4_in}) and (\ref{sb-opt1}), respectively, because the complexity difference between $cut(\cdot)$ and $shrink(\cdot)$ is a constant. Considering one pixel access as an atom operator ({\it ao}), we can compare the complexity of the two methods in detail. $\nabla_{x}$ and $\nabla_{x}^{T}$ are both $2$ {\it aos}, as are $\nabla_{y}$ and $\nabla_{x}^{T}$. According to $\Delta = -\nabla_{x}^{T} \nabla_{x} - \nabla_{y}^{T} \nabla_{y}$, $\Delta$ is $8$ {\it aos}. $\nabla_{w}$, $\Delta_{w}$ and $\text{div}_{w}$ are $20$ {\it aos} if we choose ten similar patterns in NLTV. Therefore, solving Eq. (\ref{sb-opt1}) requires $86N$ {\it aos}, while Eq. (\ref{new_step4_in}) only requires $25$ {\it aos}. So, for a large image or a large number of optimized iterations, the proposed algorithm can significantly reduce the computational time. Table \ref{cpu-compare} compares the split Bregman method and the proposed fast algorithm for solving the spatial-temporal regularizations in terms of the CPU seconds. The absence of PDEs leads to at least a one-quarter reduction in running time, and the proposed fast algorithm is therefore highly efficient.
%Õâ¶ÎÖеÄhomogeneous regionµÄ¸´Ô­ÐèÒªÔÚSpatial-Temporal Kernel Regression½øÐнøÒ»²½µÄrefine

\begin{figure}[htbp]
\setlength{\abovecaptionskip}{0pt}  %±êÌâµÄÉϱ߽ç
\setlength{\belowcaptionskip}{0pt} %±êÌâµÄϱ߽ç
\renewcommand{\figurename}{Figure}
\centering
\includegraphics[width=0.7\textwidth]{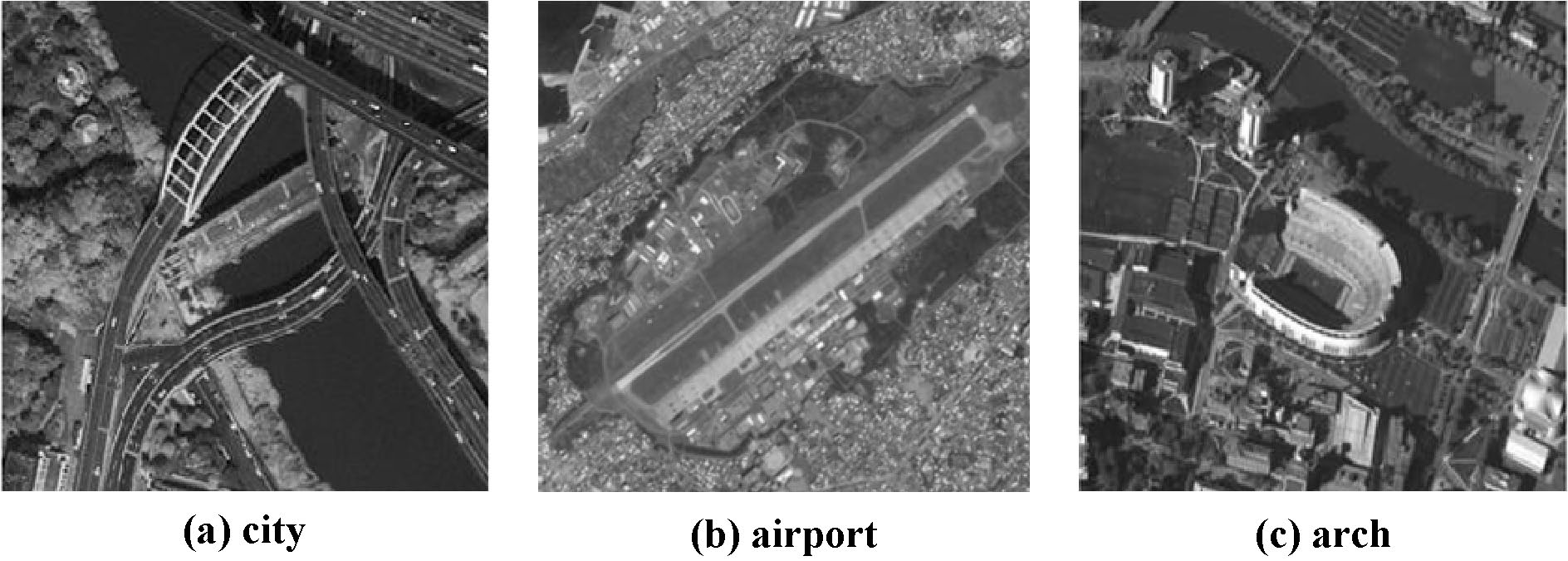}
\caption{Three sharp images used in simulated experiments.}
\label{fig:simulated_gt}
\end{figure}
\begin{table}[!htbp]
\caption{Comparison of CPU time (average) for the split Bregman and the proposed fast algorithm for differently sized images}
\centering
\label{cpu-compare}
{
\begin{tabular}{|l||r|r|r|}
\hline
 & $240\times 240$ & $260\times 260$ & $320\times 240$ \\
\hline\hline

\multirow{1}*{Split Bregman}
      & 0.6496 & 0.6231 & 1.1588 \\\hline
\multirow{1}*{Proposed}
      & 0.4482 & 0.4693 & 0.7805 \\\hline
\multirow{1}*{Reduction}
      & $31.00\%$ & $24.68\%$ & $32.65\%$ \\\hline
\end{tabular}
}%
\end{table}

\subsection{Simulated Experiments}\label{experiment-simulate}
\begin{figure*}[!htbp]
\setlength{\abovecaptionskip}{0pt}  %±êÌâµÄÉϱ߽ç
\setlength{\belowcaptionskip}{0pt} %±êÌâµÄϱ߽ç
\renewcommand{\figurename}{Figure}
\centering
\includegraphics[width=0.8\textwidth]{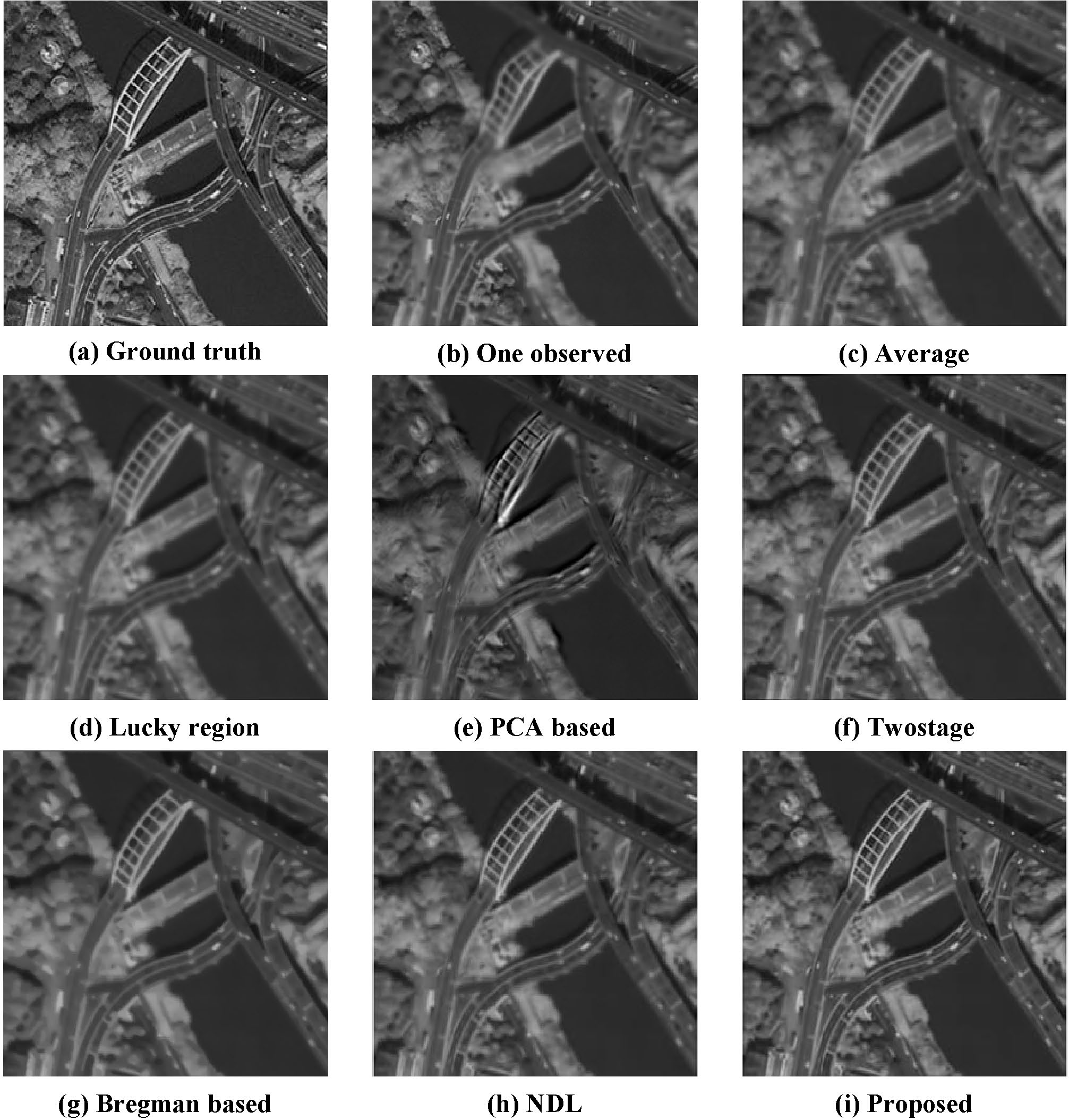}
\caption{Image restoration results on the simulated city\_weak sequence.}
\label{fig:city_weak_compare}
\end{figure*}

\begin{figure*}[htbp]
\setlength{\abovecaptionskip}{0pt}  %±êÌâµÄÉϱ߽ç
\setlength{\belowcaptionskip}{0pt} %±êÌâµÄϱ߽ç
\renewcommand{\figurename}{Figure}
\centering
\includegraphics[width=0.8\textwidth]{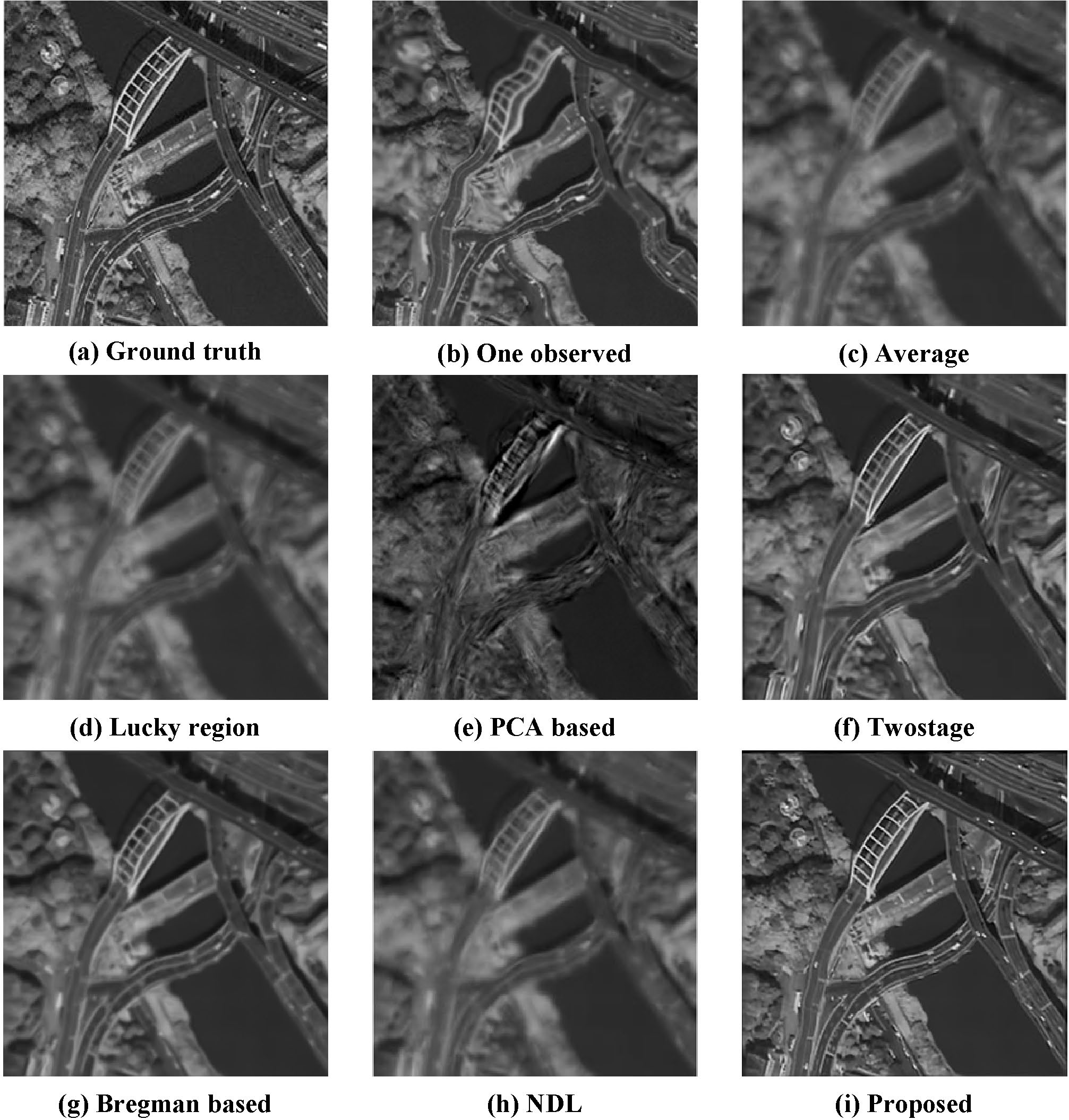}
\caption{Image restoration results on the simulated city\_strong sequence.}
\label{fig:city_strong_compare}
\end{figure*}

\begin{table*}\footnotesize
\caption{Comparison of PSNR and SSIM obtained by different methods}
\centering
\label{compare-table}
{
\begin{tabular}{l||r|r|r|r|r|r|r}
\hline
{\itshape Sequence} & Average & Lucky-Region & PCA Based & Twostage & BNLTV & NDL & Proposed\\
\hline\hline

\multirow{2}*{Airport}
      & $69.7747$ & $69.7758$ & $66.7340$ & $68.7511$ & $69.3926$ & $68.3036$ & $\mathbf{69.9754}$ \\\cline{2-8}
      & $0.2745$ & $0.2731$ & $0.2042$ & $0.2566$ & $0.3830$ & $0.3520$ & $\mathbf{0.4011}$\\\cline{2-8}
      \dhline
\multirow{2}*{City\_strong}
      & $68.2606$ & $66.9073$ & $63.2749$ & $66.4798$ & $67.0916$ & $\mathbf{68.6392}$ & $67.4820$ \\\cline{2-8}
      & $0.2479$ & $0.2405$ & $0.1329$ & $0.2378$ & $0.3413$ & $0.2854$ & $\mathbf{0.3544}$ \\\cline{2-8}
      \dhline
\multirow{2}*{City\_weak}
      & $68.8989$ & $67.3415$ & $64.8207$ & $67.4399$ & $67.3981$ & $\mathbf{68.9629}$ & $68.7218$ \\\cline{2-8}
      & $0.3611$ & $0.3531$ & $0.2215$ & $0.2686$ & $0.3631$ & $0.3956$ & $\mathbf{0.4022}$ \\\cline{2-8}
      \dhline
\multirow{2}*{Architecture}
      & $70.2634$ & $68.1599$ & $64.1939$ & $68.8656$ & $68.3711$ & $67.9243$ & $\mathbf{70.5919}$ \\\cline{2-8}
      & $0.4299$ & $0.4282$ & $0.2929$ & $0.3656$ & $0.4349$ & $0.4340$ & $\mathbf{0.5851}$ \\\cline{2-8}
      \dhline
\multirow{2}*{Chimney}
      & $78.9262$ & $79.1731$ & $64.4229$ & $78.7460$ & $75.0367$ & $79.5622$ & $\mathbf{79.9611}$ \\\cline{2-8}
      & $0.1194$ & $0.1137$ & $0.0539$ & $0.1193$ & $0.1120$ & $0.1075$ & $\mathbf{0.1257}$ \\\cline{2-8}
      \dhline
\multirow{2}*{Building}
      & $72.1991$ & $72.1987$ & $62.4768$ & $73.2032$ & $72.9388$ & $74.3453$ & $\mathbf{74.3843}$ \\\cline{2-8}
      & $0.3201$ & $0.3271$ & $0.1858$ & $0.3882$ & $0.3855$ & $0.4857$ & $\mathbf{0.5023}$ \\\cline{2-8}
      \hline
%\multirow{2}*{WaterTower}
%      & first & second & third & fourth & fifth & sixth & our \\\cline{2-8}
%      & first & second & third & fourth & fifth & sixth & our \\\hline
\end{tabular}
}%
\end{table*}

\begin{figure*}[htbp]
\setlength{\abovecaptionskip}{0pt}  %±êÌâµÄÉϱ߽ç
\setlength{\belowcaptionskip}{0pt} %±êÌâµÄϱ߽ç
\renewcommand{\figurename}{Figure}
\centering
\includegraphics[width=0.8\textwidth]{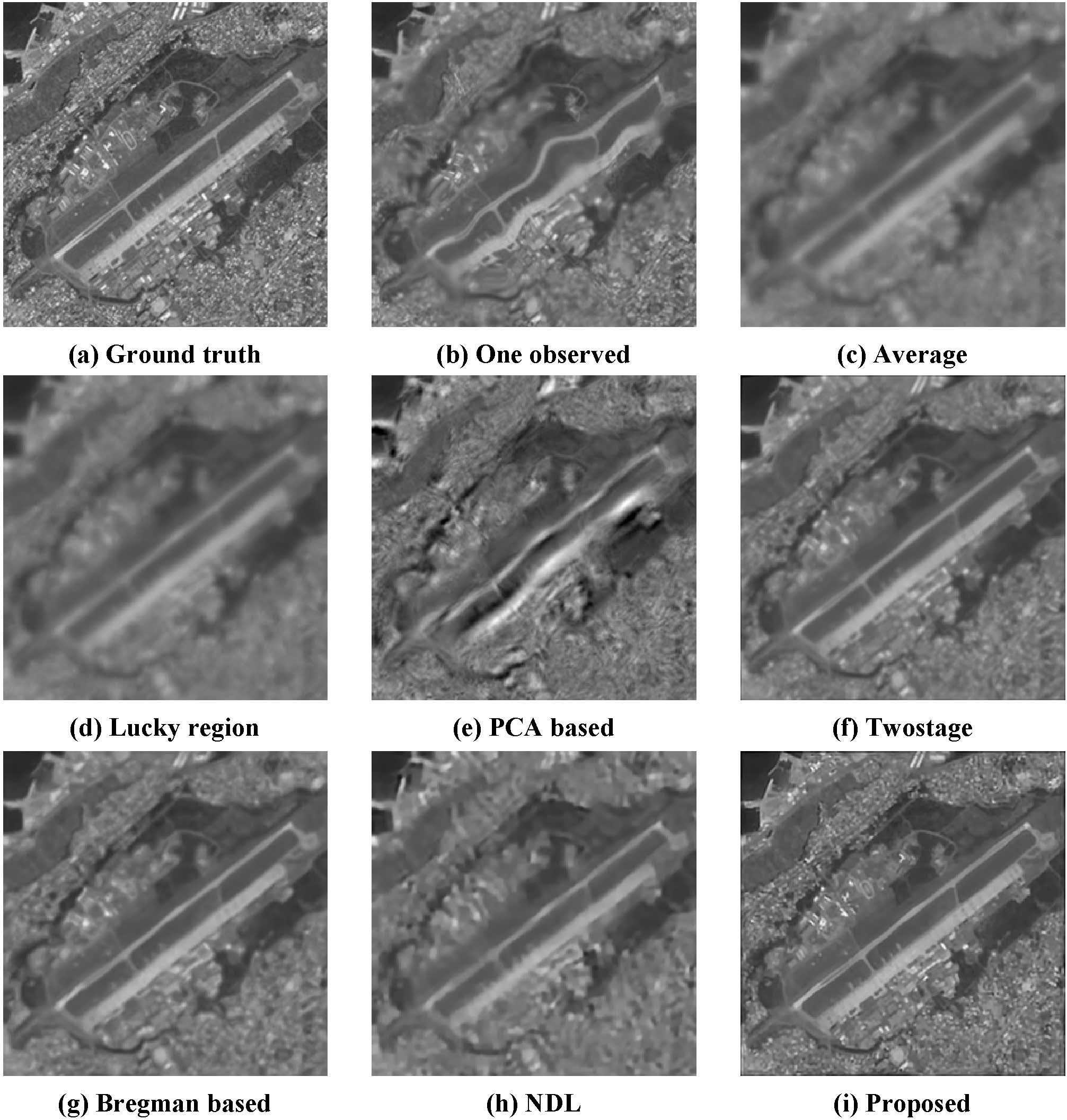}
\caption{Image restoration results on the simulated airport sequence.}
\label{fig:airport_compare}
\end{figure*}
To quantitatively evaluate the performance of the proposed method, we generated a set of simulated degraded sequences with different degrees of turbulence. The latent sharp images ($260\times 260$) are shown in Fig. \ref{fig:simulated_gt}. The simulation algorithm is similar to \cite{xzhu2} but with a slight difference. The algorithm includes three key components: the deformation field, spatially variant PSFs, and a spatially-invariant diffraction-limited PSF. The deformation field is determined by a set of control points whose random offsets have a Gaussian distribution with mean $\mu_d = 0$ and variance $\sigma_d^2$, with $\sigma_d^2$ indicating the turbulence strength. The space-varying blur is generated by convolution with a set of PSFs, each of which is also a Gaussian function with variance proportional to the magnitude of the corresponding local motion energy. In our implementation, the number of control points is characterized by their interval $d_g$, and the number of PSFs is equal to the number of control points. To efficiently apply a spatially-variant blur, we use the fast algorithm described in \cite{olv1}, which is based on overlap-add convolution schemes \cite{olv2} and linear interpolation of measured PSFs for spatially-invariant blurs. The spatially-invariant diffraction-limited blur is produced using a disc function.

To test the performance of the different restoration methods with different degrees of turbulence, we produced two degraded sequences for the same true image (Fig. \ref{fig:simulated_gt} (a)) named city\_weak and city\_strong, respectively. Considering limited page length, we only simulate two kinds of turbulence for one image (city); the other two images are used to produce the airport sequence with strong turbulence and the arch sequence with weak turbulence, respectively. The parameter configuration for the two cases are listed as follows: for weak turbulence,  $\sigma_d^{2} = 4$, $d_g = 32$, and the variance of the added Gaussian noise $\sigma_n^{2} = 3$; for strong turbulence, $\sigma_d^{2} = 10$, $d_g = 16$, and the variance of the added Gaussian noise $\sigma_n^{2} = 16$. Table \ref{compare-table} compars the PSNR and SSIM values for all the outputs of the seven different restoration algorithms. Each sequence has two sub rows, the top one denoting the PSNR values, and the bottom one denoting the SSIM values.

The restoration results of the city\_weak sequence are illustrated in Fig. \ref{fig:city_weak_compare}. The lucky region method (Fig. \ref{fig:city_weak_compare} (d)) does not seem to offer an improvement over the mean image (Fig. \ref{fig:city_weak_compare} (c)). The PCA-based method produces unnatural components due to the loss of high frequency information. BNLTV restores the structure of the objects well but cannot recover local details. This is due in large part to the smoothness caused by the weighted averaging of non-local self-similar patterns. Consequently, the smoothness prevents the restoration algorithm from recovering the local details. Twostage and NDL achieve similar results (Figs. \ref{fig:city_weak_compare} (f) and (h)); the distortion has been corrected thoroughly but some blur still exists. The proposed approach significantly improves visual quality and recovers many high-frequency details in the image. Even more noticeable differences are shown in Fig. \ref{fig:city_strong_compare}, depicting the results of the city\_strong sequence with severe turbulence (see one of the observed frames in Fig. \ref{fig:city_strong_compare} (b)). Twostage and BNLTV are superior to NDL, as shown in Fig. \ref{fig:city_strong_compare} (h), which produces many artifacts. Only the proposed method can remove large deformations and simultaneously recover detail. The superiority of the proposed algorithm is due to the high-quality reference image and its subsequently optimized version facilitating the removal of distortion, and the spatial-temporal kernel regression recovering sharp local details as well as reducing the noise introduced by the sensor and the registration error. More experimental results are shown in Fig. \ref{fig:airport_compare} and Fig. 1 in the Appendix.

\subsection{Real Video Experiments}\label{experiment-real}
\begin{figure*}[htbp]
\setlength{\abovecaptionskip}{0pt}  %±êÌâµÄÉϱ߽ç
\setlength{\belowcaptionskip}{0pt} %±êÌâµÄϱ߽ç
\renewcommand{\figurename}{Figure}
\centerline{
\begin{minipage}[b]{1.1\textwidth}
\centerline{
\includegraphics[width=\textwidth]{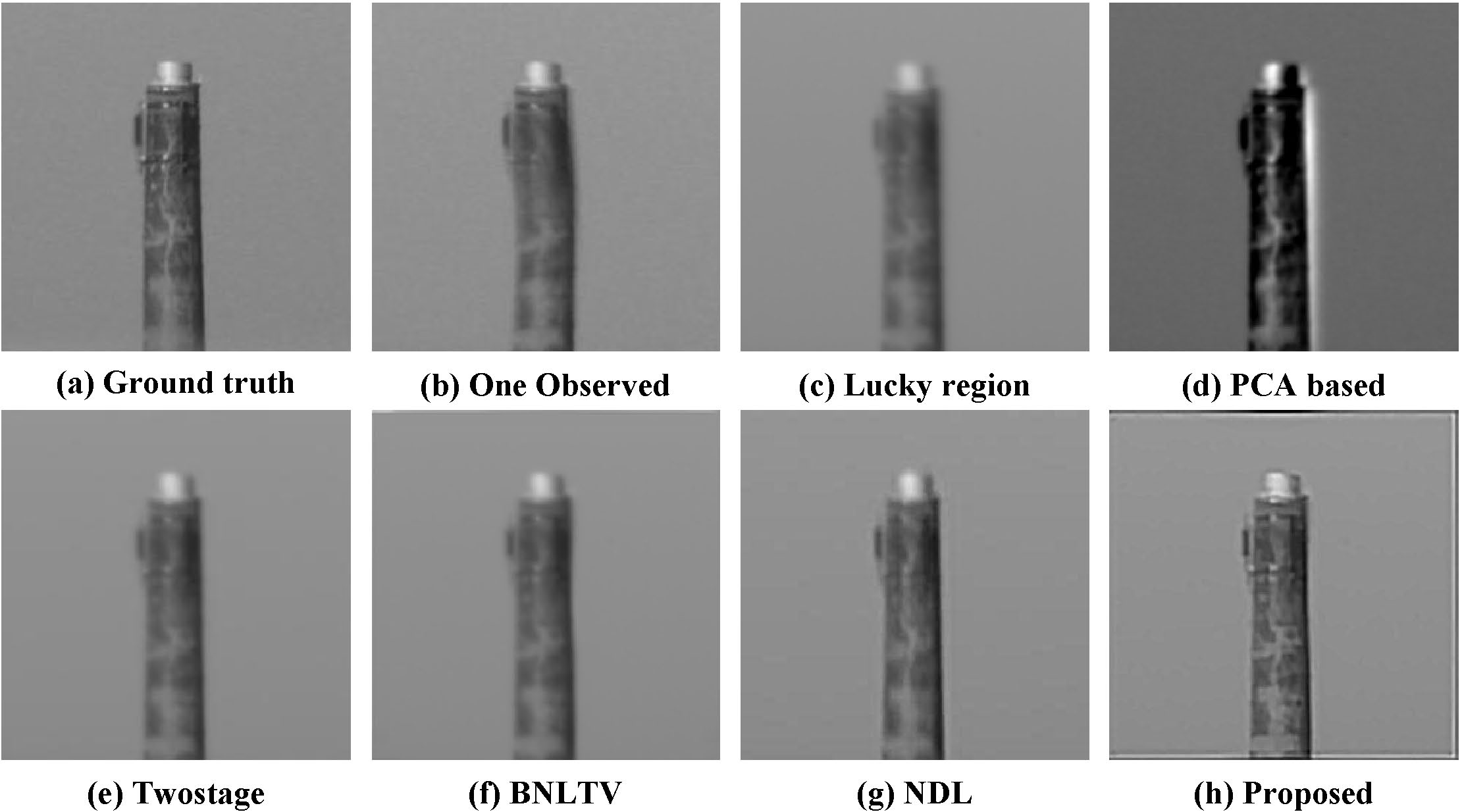}}
\caption{Image restoration results on the chimney sequence.}
\label{fig:chimney_compare}
\end{minipage}}
\end{figure*}
We test several video sequences captured through real atmospheric turbulence to illustrate the performance of the proposed restoration algorithm. In the first sequence, a chimney is captured through hot air exhausted by a building's vent (the original size of the videos 'chimney' and 'building' is $237\times237\times100$; we resize them to $240\times240\times100$ in order to facilitate non-rigid registration). The final outputs of all the methods are shown in Fig. \ref{fig:chimney_compare}: the proposed algorithm provides the best restoration result and faithfully recovers details of the object. The PSNR and SSIM values also indicate that the proposed method outperforms the other methods for the chimney sequence. A noteworthy phenomenon is that all the PSNR values of the average outputs are relatively high (shown in Table \ref{compare-table}), because PSNR is known to sometimes correlate poorly with human perception \cite{GMSD}.

Similar restoration results are seen in the experiment on the second test sequence, 'building' ($240\times240\times100$). Most restoration methods cannot remove diffraction-limited blur, except for NDL and our algorithm. Both of these methods can produce sharp restored images, but the NDL output (Fig. \ref{fig:building_compare} (g)) contains halo artifact near the edges, such as around the windows of the building. Such artifacts can be attributed to the limited accuracy of the PSF estimation and the existence of noise caused by sensor and registration error. As mentioned earlier, NDL and the proposed method utilize the same deconvolution algorithm with identical parameters to produce the final outputs. Since the major cause of artifact appears to be noise, we can conclude that the noise can be effectively reduced by spatial-temporal kernel regression. More real data experimental results are presented in the Appendix.

\begin{figure*}[htbp]
\setlength{\abovecaptionskip}{0pt}  %±êÌâµÄÉϱ߽ç
\setlength{\belowcaptionskip}{0pt} %±êÌâµÄϱ߽ç
\renewcommand{\figurename}{Figure}
\centerline{
\begin{minipage}[b]{1.1\textwidth}
\centerline{
\includegraphics[width=\textwidth]{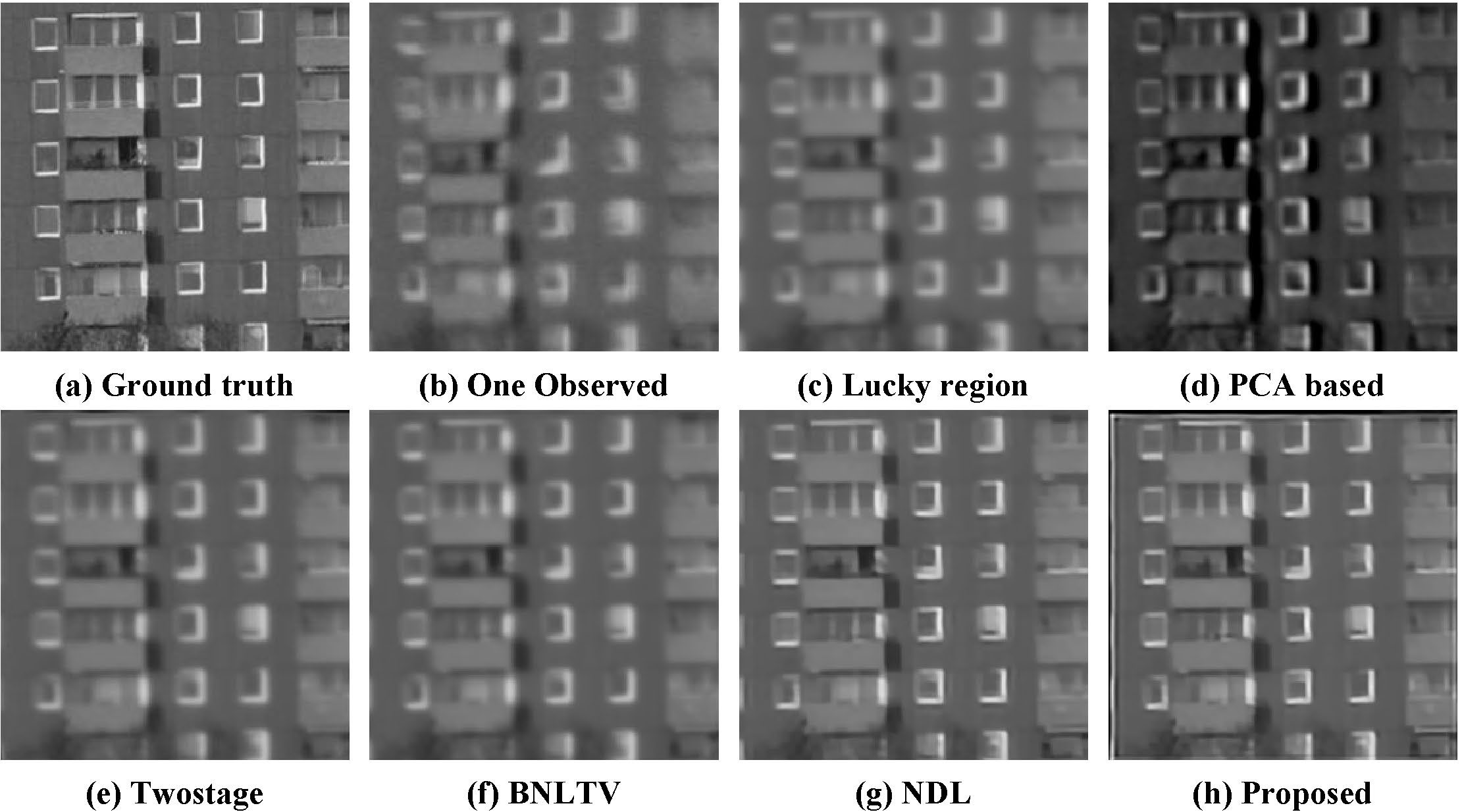}}
\caption{Image restoration results on the building sequence.}
\label{fig:building_compare}
\end{minipage}}
\end{figure*}

%\begin{figure*}[!htbp]
%\setlength{\abovecaptionskip}{0pt}  %±êÌâµÄÉϱ߽ç
%\setlength{\belowcaptionskip}{0pt} %±êÌâµÄϱ߽ç
%\renewcommand{\figurename}{Figure}
%\centering
%%\includegraphics[width=0.4\textwidth]{eijkel2}
%\fbox{\includegraphics[width=0.8\textwidth]{image/watertower_compare.eps}}
%\caption{Image restoration results on Watertower sequence.}
%\label{fig:watertower_compare}
%\end{figure*}

\section{Discussion and Conclusions}\label{discussion-and-conclusion}
In this paper, we propose a new method to restore a high-quality image from a given image sequence degraded by atmospheric turbulence. Geometric distortion and space-time-varying blur are the major challenges that need to be overcome during restoration. To effectively remove local distortion, the proposed method constructs a high-quality reference image using low-rank decomposition. Then, to further improve the registration, the proposed method applies a variational model with a novel spatial-temporal regularization term to iteratively optimize the reference image; the proposed fast algorithm can efficiently solve this variational model. To reduce blur variation, the registered frames can be fused into a single image with reduced PSF variation by applying near-stationary patch detection and distortion-driven spatial-temporal kernel regression. The fused image can be deblurred by a space-invariant blind deconvolution method in order to produce the final output. Thorough empirical studies on a set of simulated data and real videos demonstrate the effectiveness and the efficiency of the new framework for recovering a single image from a degraded sequence.

Additionally, we test the performance of the proposed restoration method on images distorted by water rather than air (see the Appendix in supplemental material), with promising results. This method is therefore suitable for use with multiple distorting media, although further adaptations of the proposed method are likely.

\section{Proof}

\subsection{Split Bregman for Solving Mixed ROF Model}\label{SB}
The split Bregman method for solving the mixed ROF subproblem (\ref{split_bregman_rof}) is presented below:

Let $b_{w}^{0} = 0, b_{x}^{0} = 0, b_{y}^{0} = 0$ and $u^{0} = v$, for $k=1,2,\ldots$, the iteration is as follows
\begin{align}\label{sb-opt1}
   u^{k+1} = &(I - 2\lambda_1 \Delta_{w} - \lambda_2\Delta)^{-1} (v + \lambda_1 \text{div}_{w} (b_{w}^{k} - d_{w}^{k}) + \lambda_2 \nabla_{x}^{T} (d_{x}^{k} + \nabla_{x} u_p - b_{x}^{k})\\ \nonumber
   &+ \lambda_2 \nabla_{y}^{T}(d_{y}^{k} + \nabla_{y} u_p - b_{y}^{k}))
\end{align}
\begin{align}\label{}
   &d_{w}^{k+1} = shrink\Big( \nabla_{w}u^{k+1} + b_{w}^{k}, \frac{\mu_1}{\lambda_1}\Big) \label{sb-opt2} \\
   &d_{x}^{k+1} = shrink\Big( \nabla_{x}u^{k+1} - \nabla_{x}u_p + b_{x}^{k}, \frac{\mu_2}{\lambda_2}\Big) \label{sb-opt3}\\
   &d_{y}^{k+1} = shrink\Big( \nabla_{y}u^{k+1} - \nabla_{y}u_p + b_{y}^{k}, \frac{\mu_2}{\lambda_2}\Big) \label{sb-opt4}
\end{align}
where $\nabla_{x}$ and $\nabla_{y}$ are the difference operators along the $x$ and $y$ directions, respectively, $\Delta := -\nabla_{x}^{T}\nabla_{x}- \nabla_{y}^{T}\nabla_{y}$ is the discrete Laplacian operator, $\nabla_{w}$ denotes the non-local gradient operator and $\Delta_{w}$ represents the non-local graph Laplacian, and $\text{div}_{w}$ is the non-local graph divergence operator (see the mathematical definitions of all those operators in Appendix). Moreover, the function $shrink(\cdot,\cdot)$ is defined as:
\begin{equation}\label{shrink_op}
 shrink(x, \gamma) = \frac{x}{|x|}\ast \max(|x| - \gamma, 0)
\end{equation}

\subsection{Proof of Theorem 1}\label{proof-B}
The detailed proof of the theorem \ref{main_theorem} is given below.
\begin{lemma}\label{lemma1}
  Given $0< 20\lambda_1 + 4\lambda_2 <1$, the real symmetric linear operator $I + 2\lambda_1\Delta_{w} + \lambda_2\Delta$ is positive definite.
\end{lemma}

\begin{lemma}\label{lemma2}
  For $k=1,2,\ldots$, let $d_{w}^{k}, d_{x}^{k}, d_{y}^{k}, b_{w}^{k}, b_{x}^{k},b_{y}^{k}$ and $u^{k+1}$ be given by the iteration (\ref{new_step1_in}) to (\ref{new_step4_in}). Then $\lim_{k\rightarrow \infty} (u^{k+1} - u^{k}) = 0$, and
  \begin{align*}\label{}
    &b_{w}^{k} = cut(\nabla_{w} u^{k} + b_{w}^{k-1}, \frac{\mu_1}{\lambda_1})\\
    &b_{x}^{k} = cut(\nabla_{x} u^{k} - \nabla_{x} u_p + b_{x}^{k-1}, \frac{\mu_2}{\lambda_2})\\
    &b_{y}^{k} = cut(\nabla_{y} u^{k} - \nabla_{y} u_p + b_{y}^{k-1}, \frac{\mu_2}{\lambda_2})
  \end{align*}
\end{lemma}

\begin{lemma} \label{lemma_3}
  For $k=1,2,\ldots$, let $d_{w}^{k}, d_{x}^{k}, d_{y}^{k}, b_{w}^{k}, b_{x}^{k},b_{y}^{k}$ and $u^{k+1}$ be given by the iteration (\ref{new_step1_in}) to (\ref{new_step4_in}). Then all the sequences $(d_{w}^{k})_{k=1,2,\ldots},(d_{x}^{k})_{k=1,2,\ldots},(d_{y}^{k})_{k=1,2,\ldots},(b_{w}^{k})_{k=1,2,\ldots},(b_{x}^{k})_{k=1,2,\ldots},(b_{y}^{k})_{k=1,2,\ldots}$ and $(u^{k})_{k=1,2,\ldots}$ are bounded. Moreover,
  \begin{align*}\label{}
    u^{k+1} = v + \frac{\lambda_1}{\mu_1} div_{w} b_{w}^{k} - \frac{\lambda_2}{\mu_2} (\nabla_{x}^{T} b_{x}^{k} + \nabla_{y}^{T} b_{y}^{k})
  \end{align*}
\end{lemma}

\begin{lemma}\label{lemma_4}
  For $k=1,2,\ldots$, let $d_{w}^{k}, d_{x}^{k}, d_{y}^{k}, b_{w}^{k}, b_{x}^{k},b_{y}^{k}$ and $u^{k+1}$ be given by the iteration (\ref{new_step1_in}) to (\ref{new_step4_in}). Then
  \begin{align*}\label{}
    &\lim_{k\rightarrow \infty} (d_{w}^{k} - \nabla_{w} u^{k}) = 0\\
    &\lim_{k\rightarrow \infty} (d_{x}^{k} - \nabla_{x} (u^{k} - u_p)) = 0\\
    &\lim_{k\rightarrow \infty} (d_{y}^{k} - \nabla_{y} (u^{k} - u_p)) = 0
  \end{align*}
\end{lemma}
\noindent The proofs of the above lemmas are given in Appendix in supplemental material. 

\begin{thmA}
For $k=1,2,\ldots$, let $b_{w}^{k}, b_{x}^{k},b_{y}^{k}$ and $u^{k+1}$ be given by the iteration (\ref{new_step1_in}) to (\ref{new_step4_in}). If $0< 20\lambda_1 + 4\lambda_2 <1$, then $\lim_{k\rightarrow \infty} u^{k} = u^{\ast}$.
\end{thmA}

\begin{proof}
  Let $F(u):= (1/2) \|u-v\|_{2}^{2}$, then $\partial F(u) = u-v$. For $w\in \mathbb{R}^{N^{2}}$ we can get
  \begin{equation}\label{}
    F(u^{k+1} + w) - F(u^{k+1}) - \langle u^{k+1}-v, w\rangle \geq 0
  \end{equation}
  From Lemma \ref{lemma_3}, we have $-(u^{k+1}-v) = \frac{\lambda_2}{\mu_2} \nabla_{x}^{T}b_{x}^{k} + \frac{\lambda_2}{\mu_2} \nabla_{y}^{T}b_{y}^{k} - \frac{\lambda_1}{\mu_1} \text{div}_{w}b_{w}^{k}$, and moreover $\langle \text{div}_{w}p, q \rangle = -\langle p, \nabla_{w} q \rangle$. Then,
  \begin{equation}\label{proof_main_1}
    F(u^{k+1} + w) - F(u^{k+1}) + \langle \frac{\lambda_2}{\mu_2} b_{x}^{k}, \nabla_{x}w \rangle + \langle \frac{\lambda_2}{\mu_2} b_{y}^{k}, \nabla_{y}w \rangle + \langle \frac{\lambda_1}{\mu_1} b_{w}^{k}, \nabla_{w}w \rangle\geq 0
  \end{equation}
  Recall that $G(d) = \|d\|_{1}$, $\frac{\lambda_1}{\mu_1} b_{w}^{k}\in \partial G(d_{w}^{k}), \frac{\lambda_2}{\mu_2} b_{x}^{k}\in \partial G(d_{x}^{k})$ and $\frac{\lambda_2}{\mu_2} b_{y}^{k}\in \partial G(d_{y}^{k})$, so,
  \begin{align}
     \|d_{w}^{k} + \nabla_{w}w\|_{1} - \|d_{w}^{k}\|_{1} - \langle \frac{\lambda_1}{\mu_1} b_{w}^{k}, \nabla_{w}w \rangle\geq 0 \label{proof_main_2}\\
     \|d_{x}^{k} + \nabla_{x}w\|_{1} - \|d_{x}^{k}\|_{1} - \langle \frac{\lambda_2}{\mu_2} b_{x}^{k}, \nabla_{x}w \rangle\geq 0 \label{proof_main_3}\\
     \|d_{y}^{k} + \nabla_{y}w\|_{1} - \|d_{y}^{k}\|_{1} - \langle \frac{\lambda_2}{\mu_2} b_{y}^{k}, \nabla_{y}w \rangle\geq 0 \label{proof_main_4}
  \end{align}
  Adding (\ref{proof_main_1}) $\sim$ (\ref{proof_main_4}) gives
  \begin{equation}\label{proof_main_5}
     \|d_{w}^{k}\|_{1} + \|d_{x}^{k}\|_{1} + \|d_{y}^{k}\|_{1} + F(u^{k+1}) \leq
     \|d_{w}^{k} + \nabla_{w}w\|_{1} + \|d_{x}^{k} + \nabla_{x}w\|_{1} + \|d_{y}^{k} + \nabla_{y}w\|_{1} + F(u^{k+1} + w)
  \end{equation}
  Suppose that $(k_j)_{j=1,2,\ldots}$ is an increasing sequence of positive integers such that the sequence $(u^{k_j})_{j=1,2,\ldots}$ converges to the limit $\tilde{u}$. By Lemma \ref{lemma2}, we have $\lim_{k\rightarrow \infty} (u^{k+1} - u^{k}) = 0$. Therefore, $\lim_{j\rightarrow \infty} u^{k_j+1} = \tilde{u}$. Moreover, we have
  the following via the Lemma \ref{lemma_4}:
  \begin{equation}\label{}
     \lim_{j\rightarrow \infty} d_w^{k} = \lim_{j\rightarrow \infty} [(d_w^{k} - \nabla_w u^{k}) + \nabla_w u^{k}]
      = \nabla_w \tilde{u}
  \end{equation}
  also have
  \begin{equation}\label{}
     \lim_{j\rightarrow \infty} d_x^{k} = \lim_{j\rightarrow \infty} [(d_x^{k} - \nabla_x (u^{k} - u_p)) + \nabla_x (u^{k} - u_p)]
      = \nabla_x (\tilde{u} - u_p)
  \end{equation}
  \begin{equation}\label{}
     \lim_{j\rightarrow \infty} d_y^{k} = \lim_{j\rightarrow \infty} [(d_y^{k} - \nabla_y (u^{k} - u_p)) + \nabla_y (u^{k} - u_p)]
      = \nabla_y (\tilde{u} - u_p)
  \end{equation}
  Replacing $k$ by $k_j$ in (\ref{proof_main_5}) and let $j\rightarrow \infty$, we have:
  \begin{align*}
    \|\nabla_{w} \tilde{u}\|_{1} + \|\nabla_{x} (\tilde{u} - u_p)\|_{1} + \|\nabla_{y} (\tilde{u} - u_p)\|_{1} + F(\tilde{u}) \leq \|\nabla_{w} (\tilde{u} + w)\|_{1} + \\
    \|\nabla_{x} (\tilde{u} - u_p + w)\|_{1} + \|\nabla_{y} (\tilde{u} - u_p + w)\|_{1} + F(\tilde{u} + w)
  \end{align*}
  The above equations hold for all the $w\in \mathbb{R}^{N^{2}}$. On the other hand, $u^{\ast}$ is the unique solution to the minimization problem (\ref{split_bregman_rof}). Therefore, we must have $\tilde{u} = u^{\ast}$. Since $(u^{k})_{k=1,2,\ldots}$ is a bounded sequence, we have
  \begin{equation}\label{}
    \lim_{k\rightarrow \infty} u^{k} = u^{\ast}.
  \end{equation}
  This completes the proof of the Main Theorem \ref{main_theorem}.

\end{proof}

% or


%\ifCLASSOPTIONcompsoc
%  % The Computer Society usually uses the plural form
%\else
%  % regular IEEE prefers the singular form
%\fi
%
%
%The authors would like to thank Mr. M. Hirsch and Dr. S. Harmeling from Max Plank Institute for Biological Cybernetics for sharing the sequences Chimney and Building, and thank Prof. M. A. Vorontsov from the Intelligent Optics Lab of the University of Maryland for sharing the video Watertower. This work was supported by ...

\ifCLASSOPTIONcaptionsoff
  \newpage
\fi

\end{document}

\title{Appendix}

\author{Yuan~Xie,~\IEEEmembership{Member,~IEEE,}
        Wensheng~Zhang,
        Dacheng~Tao,~\IEEEmembership{Senior~Member,~IEEE,}
        Wenrui~Hu,
        Yanyun~Qu,
        Hanzi~Wang,~\IEEEmembership{Senior~Member,~IEEE,}
\IEEEcompsocitemizethanks{}}


\markboth{SUBMIT TO IEEE TRANSACTIONS ON PATTERN ANALYSIS AND MACHINE INTELLIGENCE}%
{Shell \MakeLowercase{\textit{et al.}}: Bare Demo of IEEEtran.cls for Computer Society Journals}
\maketitle

% is not selected <OR> if conference mode is selected - because compsoc
% conference papers position the abstract like regular (non-compsoc)
\IEEEdisplaynotcompsoctitleabstractindextext

% \fi
\IEEEpeerreviewmaketitle

\section{Additional Experiments Results}

\begin{figure*}[htbp]
\setlength{\abovecaptionskip}{0pt}  %标题的上边界
\setlength{\belowcaptionskip}{0pt} %标题的下边界
\renewcommand{\figurename}{Figure}
\centering
\fbox{\includegraphics[width=\textwidth]{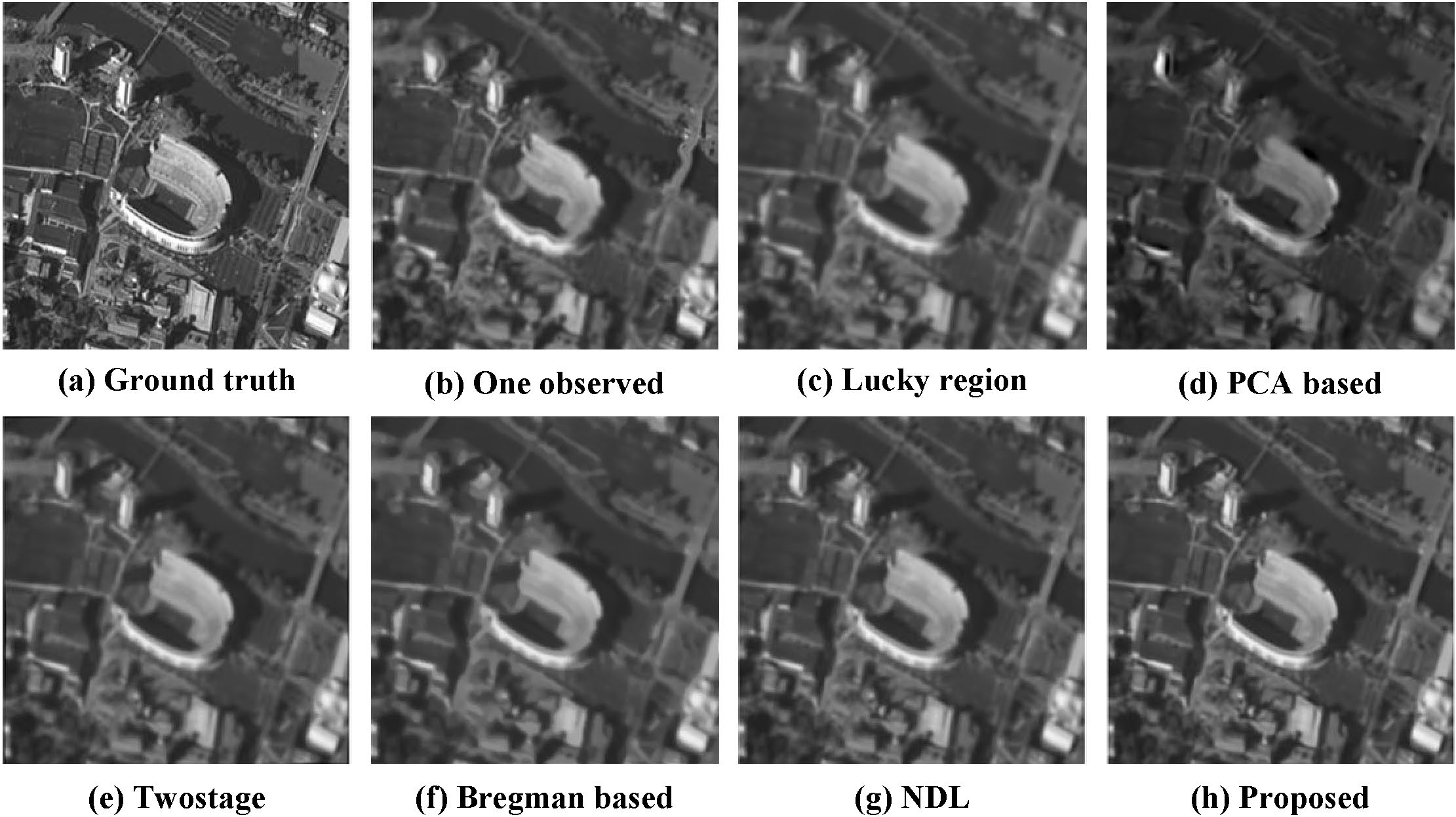}}
\caption{Image restoration results on simulated Arch sequence.}
\label{fig:arch_compare}
\end{figure*}

Another test sequence watertower shows the effect of severe turbulence during a hot summer day \footnote{This video sequence can be downloaded from the website http://www.iol.umd.edu/Movies/alot.php.}. As shown in Fig. \ref{fig:watertower_compare} (a) and (b), each frame ($320\times 240$) of the video is quite noisy and highly blurred since the video is captured by long exposure. $100$ frames are taken from the original video ($285$ frames), so the video for test is $320\times240\times100$. The lucky region, Twostage and BNLTV achieve similar results where the deformation and noise are well removed but severe blur still exists (Figs. \ref{fig:watertower_compare} (c), (e) and (f)). Only the NDL and the proposed method can give significantly improvement in visual quality. However, compared with NDL (Fig. \ref{fig:watertower_compare} (g)), the proposed method has done a bit better in handling the diffraction-limited blur (Fig. \ref{fig:watertower_compare} (h)). Since we do not have the latent sharp image of watertower sequence, the quantitative evaluation can not be provided.

\begin{figure*}[!htbp]
\setlength{\abovecaptionskip}{0pt}  %标题的上边界
\setlength{\belowcaptionskip}{0pt} %标题的下边界
\renewcommand{\figurename}{Figure}
\centering
\fbox{\includegraphics[width=0.85\textwidth]{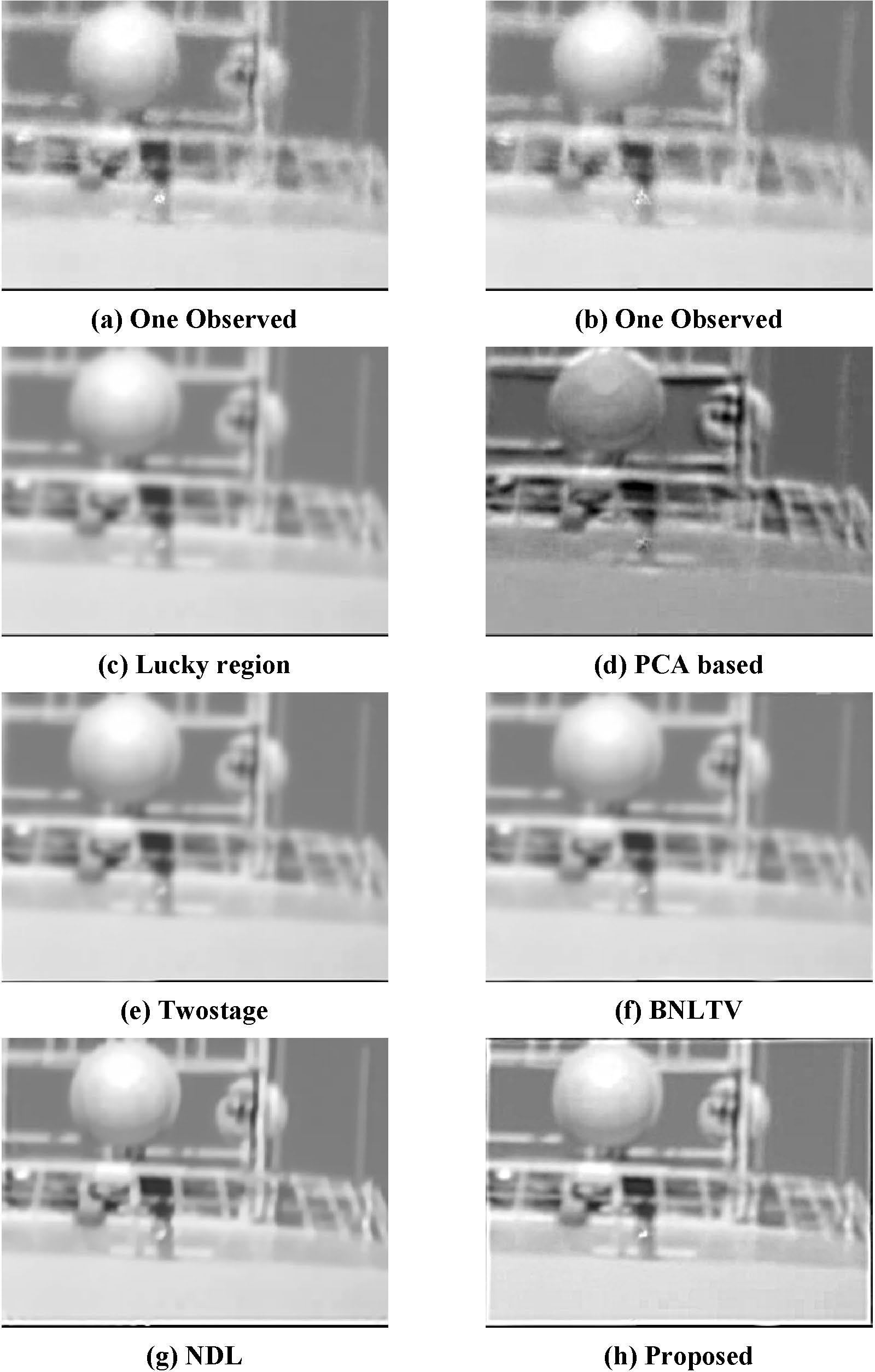}}
\caption{Image restoration results on Watertower sequence.}
\label{fig:watertower_compare}
\end{figure*}

\begin{figure*}[!htbp]
\setlength{\abovecaptionskip}{0pt}  %标题的上边界
\setlength{\belowcaptionskip}{0pt} %标题的下边界
\renewcommand{\figurename}{Figure}
\centering
\fbox{\includegraphics[width=\textwidth]{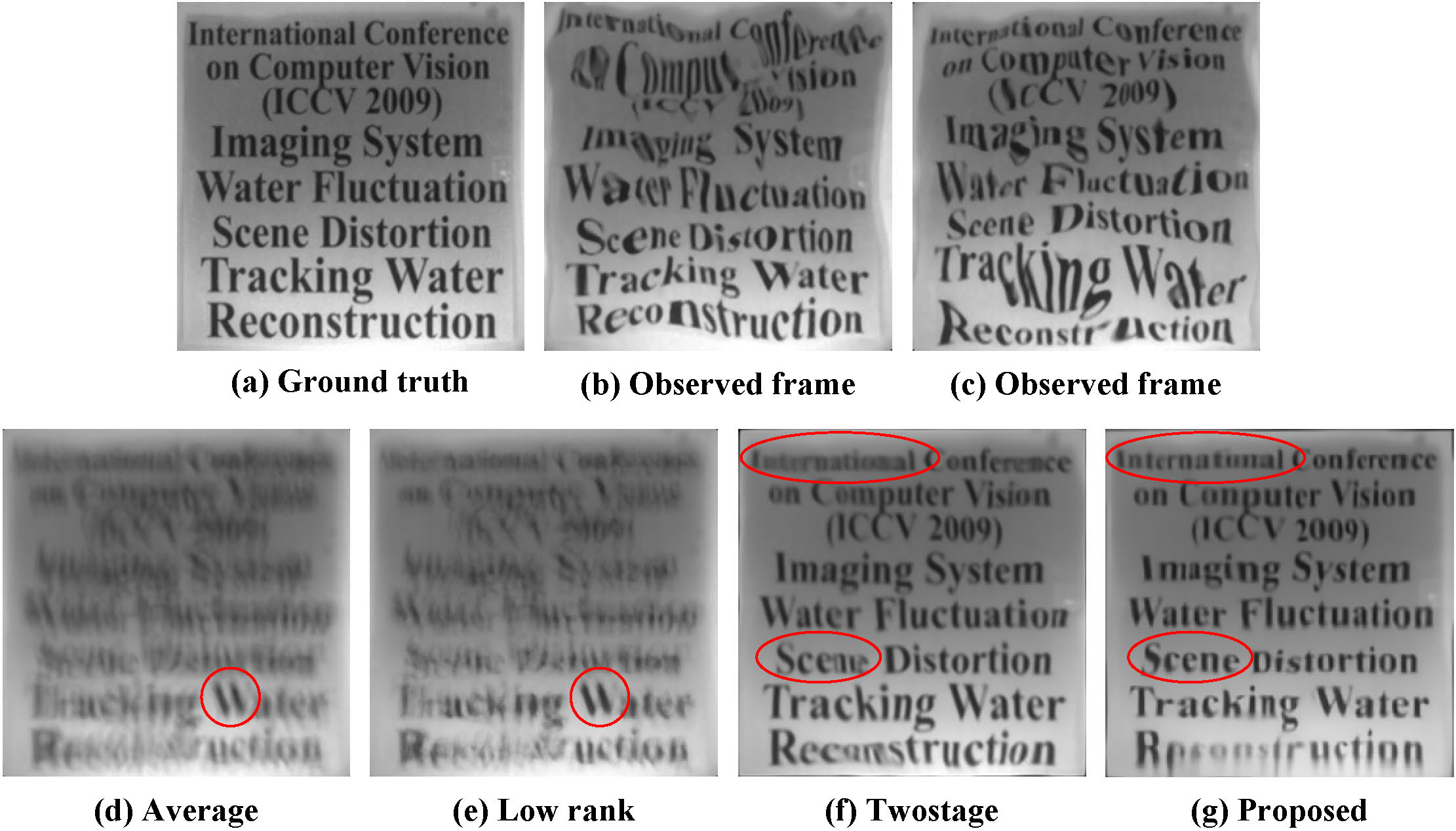}}
\caption{Image restoration results on standard underwater sequence {\it Water\_iccv} from \cite{tianyuandong}.}
\label{fig:water_iccv}
\end{figure*}

Additionally, we will discuss the performance of the proposed restoration method when the medium is changed to water. As suggested in \cite{xzhu2}, compared with atmospheric turbulence, the geometric warping effect produced by water is much stronger while the blur effect is relatively milder. Therefore, we apply the proposed method to recover the original image of an underwater scene to test whether it can handle the large geometric deformation. The Twostage method is originally designed for seeing through water, so we compare the performance between the Twostage and our method in one standard underwater sequence ($256\times 256\times 61$) from \cite{tianyuandong}. Fig. \ref{fig:water_iccv} shows the results of the two methods, where the Figs. (d) and (e) compare the quality of the reference images, and the Figs. (f) and (g) compare the visual quality of the final outputs. The circular regions in Figs. (d) and (e) demonstrate the low-rank decomposed reference image is a bit shaper than the average image. As is shown in Fig. \ref{fig:water_iccv} (g), it is clear that all the characters contained in the scene can be easily recognized except for few characters of the word ``Reconstruction'' on the lower boundary of the image. This simple test demonstrates the proposed method also can be used to tackle the {\it seeing through water} problem. However, it is still necessary to adapt the proposed method to the water case, this is our future direction.

\section{Appendix}		% Modify as required for article

\subsection{Total Variation and Nonlocal Total Variation Regularizer}\label{TVNLTV}
\paragraph{Total Variation Regularizer}
In the proposed restoration model, we impose a total variation regularizer on the difference between the current restored image $u$ and the restored $u_p$ in the previous iterative step. In other words, such regularizer will promote the sparsity in temporal domain in order to provide a smooth motion pattern between the consecutive restoration steps.
%The traditional methods only focus on enforcing the smooth in spatial domain, such as the total variation regularized term on the true image we want to restore, while the temporal information always be ignored. The intuition behind the introduction of temporal regularizer is that the motion of the pixels are similar to the ones of their neighboring pixels. Therefore, the total variation regularizer on the consecutive restored images will lead to
\begin{equation}\label{Jn}
    J_{t}(u) = |(u - u_p)|_{TV} = \| \nabla_{x}(u-u_p) \|_{1} + \| \nabla_{y}(u-u_p) \|_{1}
\end{equation}
where the subscript $x$ and $y$ represent the vertical and the horizontal component of the image domain. The difference operator $\nabla_{x}$ is given by $\nabla_{x} u(1,j)=0$ for $j=1,\ldots,N$ and
\begin{equation*}\label{}
    \nabla_{x}u(i,j) = u(i,j)-u(i-1,j), \text{ }i=2,\ldots,N,\text{ }j=1,\ldots,N.
\end{equation*}
Similarly, $\nabla_{y}$ is the difference operator given by $\nabla_{y}u(i,1)=0$ for $i=1,\ldots,N$ and
\begin{equation*}\label{}
    \nabla_{y}u(i,j) = u(i,j)-u(i,j-1), \text{ }i=1,\ldots,N,\text{ }j=2,\ldots,N.
\end{equation*}
Here, we define the conjugate operators of $\nabla_{x}$ and $\nabla_{y}$ respectively, they are $\nabla_{x}^{T}$ and $\nabla_{y}^{T}$.
\begin{equation}\label{}
    \nabla_{x}^{T} =
    \begin{cases}{}
    -u(2,j) & \text{if } i=1 \\
    u(i,j)-u(i+1,j) & \text{if } i=2,\ldots,N-1\\
    u(N,j)  & \text{if } i=N \\
 \end{cases}
\end{equation}
Similarly, the linear operator $\nabla_{y}^{T}$ is given by
\begin{equation}\label{}
    \nabla_{y}^{T} =
    \begin{cases}{}
    -u(i,2) & \text{if } j=1 \\
    u(i,j)-u(i,j+1) & \text{if } j=2,\ldots,N-1\\
    u(i,N)  & \text{if } j=N \\
 \end{cases}
\end{equation}
Let's define the discrete Laplace operator $\Delta := -\nabla_{x}^{T}\nabla_{x}- \nabla_{y}^{T}\nabla_{y}$. For $1<i,j<N$,
\begin{equation}\label{}
    -\Delta u(i,j) =
    \begin{cases}{}
    2u(1,1)-u(1,2)-u(2,1) & \text{if } i=1,j=1 \\
    3u(1,j)-u(2,j)-u(1,j-1)-u(1,j+1) & \text{if } i=1,1<j<N\\
    3u(i,1)-u(i,2)-u(i-1,1)-u(i+1,1) & \text{if } 1<i<N, j=1\\
    4u(i,j)-u(i+1,j)-u(i-1,j)-u(i,j+1)-u(i,j-1) & \text{if } 1<i<N, 1<j<N\\
 \end{cases}
\end{equation}

\paragraph{Nonlocal Total Variation Regularizer}
Suppose the digital image model by a graph $(\Omega, E)$, where $\Omega$ is a finite set of N nodes (pixels), $E$ is the set of edges. The notation $x\sim y$ is used to denote the edge between the nodes $x$ and $y$. An image $u$ is a function defined on $\Omega$, which can be represented by a column vector, then the value at node $x$ can be denoted by $u(x)$. In the following, we consider a weight function $w(x,y)$ for the edge $x\sim y \in E$. The weight function is symmetric and can be set to $0$ if two nodes $x$ and $y$ are not connected. In this case, unlike classical total variation, nodes of NLTV may directly interact with nodes that are not neighbors. That is why it called ``nonlocal''.

For a given image $u(x)$ defined on $\Omega$, the weight graph gradient $\nabla_{w} u(x)$ is defined as the vector of all directional derivatives (or edge derivative) $\nabla_{w} u(x,\cdot)$ at $x$:
\begin{equation*}\label{}
    \nabla_{w} u(x) := (\nabla_{w} u(x,y))_{y\in \Omega}
\end{equation*}
where
\begin{equation*}\label{}
    \nabla_{w} u(x,y) := (u(y) - u(x))\sqrt{w(x,y)},  \forall y\in \Omega
\end{equation*}
The directional derivatives apply to all the nodes $y$ since the weight $w(x,y)$ is extended to the whole domain $\Omega\times \Omega$. A graph divergence $div_w$ of a vector $p: \Omega\times \Omega \rightarrow R$ can be defined as follows:
\begin{equation*}\label{}
    div_{w}p(x) = \sum_{y\in \Omega} (p(x,y) - p(y,x))\sqrt{w(x,y)}
\end{equation*}
Then, the graph Laplacian is defined by:
\begin{equation*}\label{}
    \Delta_{w} u(x) := \frac{1}{2}div_{w}(\nabla_{w} u(x)) = \sum_{y\in \Omega} (u(y) - u(x))w(x,y)
\end{equation*}
Using the notations above, the nonlocal total variation can be defined as follows:
\begin{equation}\label{nltv}
    J_{s}(u) = J_{NLTV}(u) := \sum_{x\in \Omega} |\nabla_{w} u(x)| = \sum_{x\in \Omega} \sqrt{\sum_{y\in \Omega} (u(x) - u(y))^{2} w(x,y)}
\end{equation}

\subsection{Solve the Mixed-ROF Model Using Split Bregman}\label{spb_mixed_rof}
%The ROF model, despite its simple form, has proved to be very difficult to minimize by conventional methods. 1)Gradient projection based method [22] is very slow due to the non-linearity and poor conditioning of the ROF model. 2)Recently, two kinds of efficient methods ``Fixed Point Continuation''[13] and Linearized Bregman[30] were proposed. But they can not handle the BV norm which never avoid in image processing problems. 3)Split Bregman iteration which was proved to be quick convergence and numerical stability was first used in [] to solve ROF model.

The optimization problem can be expanded as follows:
\begin{equation}\label{split_bregman_rof1}
    \min_{u} \mu_{1} |\nabla_{w} u|_{1} + \mu_{2} \Big(|\nabla_{x} (u-u_{p})|_{1} + |\nabla_{y} (u-u_{p})|_{1}\Big) + \frac{1}{2} \|u - v\|_{2}^{2}
\end{equation}
To apply Bregman Splitting, we first replace $\nabla_{w} u$ by $d_w$, $\nabla_{x} (u-u_p)$ by $d_x$ and $\nabla_{y} (u-u_p)$ by $d_y$. This yields a constrained problem:
\begin{equation*}\label{}
    \min_{u} \mu_1 |d_w| + \mu_2 |d_x| + \mu_2 |d_y| + \frac{1}{2} \|u - v\|_{2}^{2}, \text{ such that } d_w = \nabla_{w} u, d_x = \nabla_{x}(u-u_p) \text{ and } d_y = \nabla_{y} (u-u_p)
\end{equation*}
Then, convert the above problem to the unconstrained problem:
\begin{equation*}\label{}
    \min_{u,d_w,d_x,d_y} \mu_1 |d_w| + \mu_2 |d_x| + \mu_2 |d_y| + \frac{1}{2} \|u - v\|_{2}^{2} +
    \frac{\lambda_1}{2} \|d_w - \nabla_{w}u\|_{2}^{2} + \frac{\lambda_2}{2} \|d_x - \nabla_{x}(u-u_p)\|_{2}^{2} +
    \frac{\lambda_2}{2} \|d_y - \nabla_{y}(u-u_p)\|_{2}^{2}
\end{equation*}
By applying the Bregman iteration:
\begin{equation}\label{mixed_sb}
    \min_{u,d_w,d_x,d_y} \mu_1 |d_w| + \mu_2 |d_x| + \mu_2 |d_y| + \frac{1}{2} \|u - v\|_{2}^{2} +
    \frac{\lambda_1}{2} \|d_w - \nabla_{w}u - b_{w}^{k}\|_{2}^{2} + \frac{\lambda_2}{2} \|d_x - \nabla_{x}(u - u_p) -b_{x}^{k}\|_{2}^{2} +
    \frac{\lambda_2}{2} \|d_y - \nabla_{y} (u - u_p) - b_{y}^{k}\|_{2}^{2}
\end{equation}
Then the above problem is equivalent to
\begin{equation}\label{mixed_sb1}
\begin{aligned}
    (u^{k+1}, d_{w}^{k+1}, d_{x}^{k+1}, d_{y}^{k+1})& = \argmin_{u,d_w,d_x,d_y} \mu_1 |d_w| + \mu_2 |d_x| + \mu_2 |d_y| + \frac{1}{2} \|u - v\|_{2}^{2} + \frac{\lambda_1}{2} \|d_w - \nabla_{w}u - b_{w}^{k}\|_{2}^{2}\\
    &+ \frac{\lambda_2}{2} \|d_x - \nabla_{x}(u - u_p) -b_{x}^{k}\|_{2}^{2} +
    \frac{\lambda_2}{2} \|d_y - \nabla_{y} (u - u_p) - b_{y}^{k}\|_{2}^{2}\\
    b_{w}^{k+1}& = b_{w}^{k} + \nabla_{w}u^{k+1} - d_{w}^{k+1}\\
    b_{x}^{k+1}& = b_{x}^{k} + \nabla_{x}u^{k+1} - \nabla_{x}u_p - d_{x}^{k+1}\\
    b_{y}^{k+1}& = b_{y}^{k} + \nabla_{y}u^{k+1} - \nabla_{y}u_p - d_{y}^{k+1}\\
\end{aligned}
\end{equation}
The solution of (\ref{mixed_sb1}) is obtained by performing an alternative minimizing process:
\begin{align}\label{}
    u^{k+1} = &\argmin_{u} \frac{1}{2} \|u - v\|_{2}^{2} + \frac{\lambda_1}{2} \|d_w - \nabla_{w}u - b_{w}^{k}\|_{2}^{2} \label{sub_u}\\
    &+ \frac{\lambda_2}{2} \|d_x - \nabla_{x}(u - u_p) -b_{x}^{k}\|_{2}^{2} +
    \frac{\lambda}{2} \|d_y - \nabla_{y} (u - u_p) - b_{y}^{k}\|_{2}^{2} \notag \\
    d_{w}^{k+1} = &\argmin_{d_w} \mu_1 |d_w| + \frac{\lambda_1}{2} \|d_w - \nabla_{w}u^{k+1} - b_{w}^{k}\|_{2}^{2}\\
    d_{x}^{k+1} = &\argmin_{d_x} \mu_2 |d_x| + \frac{\lambda_2}{2} \|d_y - \nabla_{x}(u^{k+1} - u_p) - b_{x}^{k}\|_{2}^{2}\\
    d_{y}^{k+1} = &\argmin_{d_y} \mu_2 |d_y| + \frac{\lambda_2}{2} \|d_y - \nabla_{y}(u^{k+1} - u_p) - b_{y}^{k}\|_{2}^{2}
\end{align}
Now, the subproblem (\ref{sub_u}) for $u^{k+1}$ consists in solving the linear system:

\begin{equation}\label{}
\begin{aligned}
   (I - 2\lambda_1 \Delta_{w} - \lambda_2 \Delta)u^{k+1} = &v + \lambda_1 \text{div}_{w} (b_{w}^{k} - d_{w}^{k}) + \lambda_2 \nabla_{x}^{T} (d_{x}^{k} + \nabla_{x} u_p - b_{y}^{k})\\
   &+ \lambda_2 \nabla_{y}^{T}(d_{y}^{k} + \nabla_{y} u_p - b_{y}^{k})
\end{aligned}
\end{equation}

\begin{equation}\label{split_bregman_iter1}
\begin{aligned}
   u^{k+1} = &(I - 2\lambda_1 \Delta_{w} - \lambda_2\Delta)^{-1} (v + \lambda_1 \text{div}_{w} (b_{w}^{k} - d_{w}^{k}) + \lambda_2 \nabla_{x}^{T} (d_{x}^{k} + \nabla_{x} u_p - b_{y}^{k})\\
   &+ \lambda_2 \nabla_{y}^{T}(d_{y}^{k} + \nabla_{y} u_p - b_{y}^{k}))
\end{aligned}
\end{equation}

\begin{align}\label{}
   &d_{w}^{k+1} = shrink\Big( \nabla_{w}u^{k+1} + b_{w}^{k}, \frac{\mu_1}{\lambda_1}\Big)\\
   &d_{x}^{k+1} = shrink\Big( \nabla_{x}u^{k+1} - \nabla_{x}u_p + b_{x}^{k}, \frac{\mu_2}{\lambda_2}\Big)\\
   &d_{y}^{k+1} = shrink\Big( \nabla_{y}u^{k+1} - \nabla_{y}u_p + b_{y}^{k}, \frac{\mu_2}{\lambda_2}\Big)
\end{align}

\subsection{Solve the Mixed-ROF Model without PDE}\label{spb_mixed_rof_nopde}
We will propose a new iteration schema without involving PDE computing in order to find the unique solution $u^{\ast}$ for the minimization problem:
\begin{equation}\label{split_bregman_rof1}
    \min_{u} \mu_1|\nabla_{w} u|_{1} + \mu_2 |\nabla_{x} (u-u_p)|_{1} + \mu_2 |\nabla_{y} (u-u_p)|_{1} + \frac{1}{2} \|u - v\|_{2}^{2}
\end{equation}
Let $b_{w}^{0} = 0, b_{x}^{0} = 0, b_{y}^{0} = 0$ and $u^{1} = v$, for $k=1,2,\ldots$, the iteration is following:
\begin{align}\label{}
   &b_{w}^{k} = cut(\nabla_{w} u^{k} + b_{w}^{k-1}, \frac{\mu_1}{\lambda_1})\label{new_step1}\\
   &b_{x}^{k} = cut(\nabla_{x} u^{k} - \nabla_{x} u_p + b_{x}^{k-1}, \frac{\mu_2}{\lambda_2})\label{new_step2}\\
   &b_{y}^{k} = cut(\nabla_{y} u^{k} - \nabla_{y} u_p + b_{y}^{k-1}, \frac{\mu_2}{\lambda_2})\label{new_step3}\\
   &u^{k+1} = v + \frac{\lambda_1}{\mu_1} div_{w} b_{w}^{k} - \frac{\lambda_2}{\mu_2} (\nabla_{x}^{T} b_{x}^{k} + \nabla_{y}^{T} b_{y}^{k}) \label{new_step4}
\end{align}
where the function $cut(\cdot)$ defined as follows:
\begin{equation}\label{}
    cut(c, 1/\lambda) =
    \begin{cases}{}
    1/\lambda & \text{for } c>1/\lambda \\
    c & \text{for } -1/\lambda \leq c \leq 1/\lambda\\
    -1/\lambda  & \text{for } c<-1/\lambda \\
 \end{cases}
\end{equation}

\begin{lemma}\label{lemma1}
  Given $0< 20\lambda_1 + 4\lambda_2 <1$, the real symmetric linear operator $I + 2\lambda_1\Delta_{w} + \lambda_2\Delta$ is positive definite.
\end{lemma}
\begin{proof}
    Suppose that $\eta$ is an eigenvalue of the operator $I + 2\lambda_1\Delta_{w} + \lambda_2\Delta$. Then, we want to validate $\eta>0$ when $0< 20\lambda_1 + 4\lambda_2 <1$. There exists a nonzero vector $u\in \mathbb{R}^{N^2}$ such that $(I + 2\lambda_1\Delta_{w} + \lambda_2\Delta)u = \eta u$. It follows that $(I - \eta)u = -2\lambda_1\Delta_{w}u - \lambda_2 \Delta u$. Let $m = \|u\|_{\infty} = \max_{1\leq i,j \leq N}|u(i,j)|$. There exist $i_0, j_0\in \{1,\ldots,N\}$ such that $|u(i_0,j_0)| = m$. If $\eta\leq 0 $, then $|(1 - \eta)u(i_0,j_0)|\geq m$. On the other hand, $|-2\lambda_1\Delta_{w}u(i_0,j_0) - \lambda_2 \Delta u(i_0,j_0))|\leq (20\lambda_1 + 4\lambda_2)m$ (In practice, we choose $10$ neighbors for NLTV and $4$ neighbors for TV in image domain). Hence, $m \leq (20\lambda_1 + 4\lambda_2)m$. Since $0< 20\lambda_1 + 4\lambda_2 <1$, it is obvious that $m=0$ which means $u=0$. Therefore, any eigenvalue of $(\mu + 2\lambda\Delta_{w} + \lambda\Delta)$ is positive. The $I + 2\lambda_1\Delta_{w} + \lambda_2\Delta$ is positive definite.
\end{proof}

%\begin{lemma}
%  Given $0< 20\lambda_1 + 4\lambda_2 <1$, the real symmetric linear operator $I + 2\lambda_1\Delta_{w} + \lambda_2\Delta$ is positive definite.
%\end{lemma}
%\begin{proof}
%    Suppose that $\eta>0$ is an eigenvalue of the operator $0< \lambda/\mu <1/28$. There exists a nonzero vector $u\in \mathbb{R}^{N^2}$ such that $(\mu + 2\lambda\Delta_{w} + \lambda\Delta)u = \eta u$. It follows that $(\mu - \eta)u = -\lambda(2\Delta_{w} + \Delta)$. Let $m = \|u\|_{\infty} = \max_{1\leq i,j \leq N}|u(i,j)|$. There exist $i_0, j_0\in \{1,\ldots,N\}$ such that $|u(i_0,j_0)| = m$. If $\eta\leq 0 $, then $|(\mu - \eta)u(i_0,j_0)|> \mu m$. On the other hand, $|-\lambda(2\Delta_{w}u(i_0,j_0) + \Delta u(i_0,j_0))|\leq 28\lambda m$ (In practice, we choose $10$ neighbors for NLTV and $4$ neighbors for TV in image domain). Hence, $\mu m \leq 28\lambda m$. Since $0< \lambda/\mu <1/28$, it is obvious that $m=0$ which means $u=0$. Therefore, any eigenvalue of $(\mu + 2\lambda\Delta_{w} + \lambda\Delta)$ is positive. The $\mu + 2\lambda\Delta_{w} + \lambda\Delta$ is positive definite.
%\end{proof}

We would demonstrate that the algorithm given by (\ref{new_step1}), (\ref{new_step2}), (\ref{new_step3}) and (\ref{new_step4}) has the equivalent formulation described as follows:

Let $d_{w}^{k} = d_{x}^{k} = d_{y}^{k} = 0$, $b_{w}^{k} = b_{x}^{k} = b_{y}^{k} = 0$, $u^{1} = v$. For $k = 1,2,\ldots$, let
\begin{align}
   d_{w}^{k} &= \argmin_{d_w} \{ \mu_1\|d_w\|_{1} + \frac{\lambda_1}{2} \|d_w - \nabla_{w}u^{k} - b_{w}^{k-1}\| \} \label{equiv_problem1}\\
   d_{x}^{k} &= \argmin_{d_x} \{ \mu_2\|d_x\|_{1} + \frac{\lambda_2}{2} \|d_x - \nabla_{x} (u^{k} - u_p) - b_{x}^{k-1}\| \} \label{equiv_problem2}\\
   d_{y}^{k} &= \argmin_{d_y} \{ \mu_2\|d_y\|_{1} + \frac{\lambda_2}{2} \|d_y - \nabla_{y} (u^{k} - u_p) - b_{y}^{k-1}\| \} \label{equiv_problem3}\\
   b_{w}^{k} &= \nabla_{w} u^{k} + b_{w}^{k-1} - d_{w}^{k} \label{equiv_problem4}\\
   b_{x}^{k} &= \nabla_{x} u^{k} - \nabla_{x} u_p + b_{x}^{k-1} - d_{x}^{k} \label{equiv_problem5}\\
   b_{y}^{k} &= \nabla_{y} u^{k} - \nabla_{y} u_p + b_{y}^{k-1} - d_{y}^{k} \label{equiv_problem6}\\
   u^{k+1} = &\argmin_{u} \{ \frac{1}{2}\|B(u-v)\|_{2}^{2} - \langle B^{2}(u^{k} - v), u-u^{k}\rangle + \frac{\lambda_1}{2\mu_1}\|d_{w}^{k} - \nabla_{w}u\|_{2}^{2} \label{equiv_problem7}\\
   \notag &+ \frac{\lambda_2}{2\mu_2}\|d_{x}^{k} - \nabla_{x}u + \nabla_{x}u_p\|_{2}^{2} + \frac{\lambda_2}{2\mu_2}\|d_{y}^{k} - \nabla_{y}u + \nabla_{y}u_p\| \}
\end{align}
where $B = I + 2\lambda_1\Delta_{w} + \lambda_2\Delta$. Actually, the value of $u_p$ depends on the iteration $k$ rather than the the value of $u$.

\begin{lemma}\label{lemma2}
  For $k=1,2,\ldots$, let $d_{w}^{k}, d_{x}^{k}, d_{y}^{k}, b_{w}^{k}, b_{x}^{k},b_{y}^{k}$ and $u^{k+1}$ be given by the iteration (\ref{new_step1}) to (\ref{new_step4}). Then $\lim_{k\rightarrow \infty} (u^{k+1} - u^{k}) = 0$, and
  \begin{align*}\label{}
    &b_{w}^{k} = cut(\nabla_{w} u^{k} + b_{w}^{k-1}, \frac{\mu_1}{\lambda_1})\\
    &b_{x}^{k} = cut(\nabla_{x} u^{k} - \nabla_{x} u_p + b_{x}^{k-1}, \frac{\mu_2}{\lambda_2})\\
    &b_{y}^{k} = cut(\nabla_{y} u^{k} - \nabla_{y} u_p + b_{y}^{k-1}, \frac{\mu_2}{\lambda_2})
  \end{align*}
\end{lemma}

\begin{proof}
    According to (\ref{equiv_problem1}) $\sim$ (\ref{equiv_problem3}), we can obtain the following:
    \begin{align}\label{}
    &d_{w}^{k} = shrink\Big( \nabla_{w}u^{k} + b_{w}^{k-1}, \frac{\mu_1}{\lambda_1}\Big) \label{shrink1}\\
    &d_{x}^{k} = shrink\Big( \nabla_{x}u^{k} - \nabla_{x}u_p + b_{x}^{k-1}, \frac{\mu_2}{\lambda_2}\Big)\label{shrink2}\\
    &d_{y}^{k} = shrink\Big( \nabla_{y}u^{k} - \nabla_{y}u_p + b_{y}^{k-1}, \frac{\mu_2}{\lambda_2}\Big)\label{shrink3}
    \end{align}
    Combining the (\ref{shrink1}) with (\ref{equiv_problem4}):
    \begin{align*}\label{}
        b_{w}^{k} &= \nabla_{w} u^{k} + b_{w}^{k-1} - d_{w}^{k} \\
        &= \nabla_{w} u^{k} + b_{w}^{k-1} - shrink(\nabla_{w} u^{k} + b_{w}^{k-1}, \frac{\mu_1}{\lambda_1})\\
        &= cut(\nabla_{w} u^{k} + b_{w}^{k-1}, \frac{\mu_1}{\lambda_1})
    \end{align*}
    Similarly, we can get the Eqn.\ref{new_step2} and \ref{new_step3}. Therefore, $\|b_{w}^{k}\|_{\infty} \leq \mu_1/\lambda_1$, $\|b_{x}^{k}\|_{\infty} \leq \mu_2/\lambda_2$ and $\|b_{y}^{k}\|_{\infty} \leq \mu_2/\lambda_2$ for $k=1,2,\ldots$.

    Suppose $G(d):=\|d\|_{1}$ for $d\in \mathbb{R}^{N^{2}}$. let $g_{w}^{k} = \frac{\lambda_1}{\mu_1} b_{w}^{k}$, $g_{x}^{k} = \frac{\lambda_2}{\mu_2} b_{x}^{k}$ and $g_{y}^{k} = \frac{\lambda_2}{\mu_2} b_{y}^{k}$. From the (\ref{equiv_problem1}) and (\ref{equiv_problem4}), we can get
    \begin{equation*}\label{}
        g_{w}^{k} = g_{w}^{k-1} - \frac{\lambda_1}{\mu_1} (d_{w}^{k} - \nabla_{w}u^{k}) = -\frac{\lambda_1}{\mu_1}(d_{w}^{k} - \nabla_{w}u^{k} - b_{w}^{k-1})\in \partial G(d_{w}^{k})
    \end{equation*}
    So, $g_{w}^{k} - \frac{\lambda_1}{\mu_1} (d_{w}^{k+1} - \nabla_{w}u^{k+1})\in \partial G(d_{w}^{k+1})$ and
    \begin{equation}\label{proof_2_1}
        d_{w}^{k+1} = \argmin_{d_{w}} \{ \|d_{w}\|_{1} - \langle g_{w}^{k}, d_{w} - d_{w}^{k} \rangle + \frac{\lambda_1}{2\mu_1} \| d_{w} - \nabla_{w} u^{k+1} \|_{2}^{2} \}
    \end{equation}
    The same to
    \begin{align}
        d_{x}^{k+1} = \argmin_{d_{x}} \{ \|d_{x}\|_{1} - \langle g_{x}^{k}, d_{x} - d_{x}^{k} \rangle + \frac{\lambda_2}{2\mu_2} \| d_{x} - \nabla_{x} u^{k+1} + \nabla_{x} u_p \|_{2}^{2} \} \label{proof_2_2}\\
        d_{y}^{k+1} = \argmin_{d_{y}} \{ \|d_{y}\|_{1} - \langle g_{y}^{k}, d_{y} - d_{y}^{k} \rangle + \frac{\lambda_2}{2\mu_2} \| d_{y} - \nabla_{y} u^{k+1} + \nabla_{y} u_p \|_{2}^{2} \} \label{proof_2_3}
    \end{align}
    It follows from (\ref{proof_2_1}), (\ref{proof_2_2}) and (\ref{proof_2_3}) that
    \begin{align*}
        \|d_{w}^{k+1}\|_{1} - \langle g_{w}^{k}, d_{w}^{k+1} - d_{w}^{k} \rangle + \frac{\lambda_1}{2\mu_1} \| d_{w}^{k+1} - \nabla_{w} u^{k+1} \|_{2}^{2} \leq \|d_{w}^{k}\|_{1} + \frac{\lambda_1}{2\mu_1} \| d_{w}^{k} - \nabla_{w} u^{k+1} \|_{2}^{2} \\
        \|d_{x}^{k+1}\|_{1} - \langle g_{x}^{k}, d_{x}^{k+1} - d_{x}^{k} \rangle + \frac{\lambda_2}{2\mu_2} \| d_{x}^{k+1} - \nabla_{x} u^{k+1} + \nabla_{x} u_p \|_{2}^{2} \leq \|d_{x}^{k}\|_{1} + \frac{\lambda_2}{2\mu_2} \| d_{x}^{k} - \nabla_{x} u^{k+1} + \nabla_{x} u_p \|_{2}^{2}\\
        \|d_{y}^{k+1}\|_{1} - \langle g_{y}^{k}, d_{y}^{k+1} - d_{y}^{k} \rangle + \frac{\lambda_2}{2\mu_2} \| d_{y}^{k+1} - \nabla_{y} u^{k+1} + \nabla_{y} u_p \|_{2}^{2} \leq \|d_{y}^{k}\|_{1} + \frac{\lambda_2}{2\mu_2} \| d_{y}^{k} - \nabla_{y} u^{k+1} + \nabla_{y} u_p \|_{2}^{2}
    \end{align*}
    Since $g_{w}^{k}\in \partial G(d_{w}^{k})$, $g_{x}^{k}\in \partial G(d_{x}^{k})$ and $g_{y}^{k}\in \partial G(d_{y}^{k})$, by the definition of the Bregman Distance we can get
    \begin{align}
        \frac{\lambda_1}{2\mu_1} \| d_{w}^{k+1} - \nabla_{w} u^{k+1} \|_{2}^{2} \leq \frac{\lambda_1}{2\mu_1} \| d_{w}^{k} - \nabla_{w} u^{k+1} \|_{2}^{2} \label{proof_2_4} \\
        \frac{\lambda_2}{2\mu_2} \| d_{x}^{k+1} - \nabla_{x} u^{k+1} + \nabla_{x} u_p \|_{2}^{2} \leq \frac{\lambda_2}{2\mu_2} \| d_{x}^{k} - \nabla_{x} u^{k+1} + \nabla_{x} u_p \|_{2}^{2}\label{proof_2_5}\\
        \frac{\lambda_2}{2\mu_2} \| d_{y}^{k+1} - \nabla_{y} u^{k+1} + \nabla_{y} u_p \|_{2}^{2} \leq \frac{\lambda_2}{2\mu_2} \| d_{y}^{k} - \nabla_{y} u^{k+1} + \nabla_{y} u_p \|_{2}^{2} \label{proof_2_6}
    \end{align}
    By (\ref{equiv_problem7}) we achieve the following inequality
    \begin{align*}
        &\frac{1}{2}\|B(u^{k+1}-v)\|_{2}^{2} - \langle B^{2}(u^{k} - v), u^{k+1}-u^{k}\rangle + \frac{\lambda_1}{2\mu_1}\|d_{w}^{k} - \nabla_{w}u^{k+1}\|_{2}^{2} + \frac{\lambda_2}{2\mu_2}\|d_{x}^{k} - \nabla_{x}u^{k+1} + \nabla_{x}u_p\|_{2}^{2} \\
        &+\frac{\lambda_2}{2\mu_2}\|d_{y}^{k} - \nabla_{y}u^{k+1} + \nabla_{y}u_p\|_{2}^{2} \leq \frac{1}{2}\|B(u-v)\|_{2}^{2} - \langle B^{2}(u^{k} - v), u-u^{k}\rangle + \frac{\lambda_1}{2\mu_1}\|d_{w}^{k} - \nabla_{w}u\|_{2}^{2} \\
        &+ \frac{\lambda_2}{2\mu_2}\|d_{x}^{k} - \nabla_{x}u + \nabla_{x}u_p\|_{2}^{2} +\frac{\lambda_2}{2\mu_2}\|d_{y}^{k} - \nabla_{y}u + \nabla_{y}u_p\|_{2}^{2}
    \end{align*}
    Choosing $u = u^{k}$ and replacing the $u_p$ according to $u$ in the above inequality, we obtain
    \begin{align*}
        &\frac{1}{2}\|B(u^{k+1}-v)\|_{2}^{2} - \langle B(u^{k} - v), B(u^{k+1}-u^{k})\rangle + \frac{\lambda_1}{2\mu_1}\|d_{w}^{k} - \nabla_{w}u^{k+1}\|_{2}^{2} + \frac{\lambda_2}{2\mu_2}\|d_{x}^{k} - \nabla_{x}u^{k+1} + \nabla_{x}u^{k}\|_{2}^{2} \\
        &+\frac{\lambda_2}{2\mu_2}\|d_{y}^{k} - \nabla_{y}u^{k+1} + \nabla_{y}u^{k}\|_{2}^{2} \leq \frac{1}{2}\|B(u^{k}-v)\|_{2}^{2} + \frac{\lambda_1}{2\mu_1}\|d_{w}^{k} - \nabla_{w}u^{k}\|_{2}^{2} \\
        &+ \frac{\lambda_2}{2\mu_2}\|d_{x}^{k} - \nabla_{x}u^{k} + \nabla_{x}u^{k-1}\|_{2}^{2} +\frac{\lambda_2}{2\mu_2}\|d_{y}^{k} - \nabla_{y}u^{k} + \nabla_{y}u^{k-1}\|_{2}^{2}
    \end{align*}
    Since
    \begin{align*}
        \frac{1}{2}\|B(u^{k+1}-v)\|_{2}^{2} - \frac{1}{2}\|B(u^{k}-v)\|_{2}^{2} - \langle B(u^{k} - v), B(u^{k+1}-u^{k}) \rangle = \frac{1}{2}\|B(u^{k+1}-u^{k})\|_{2}^{2}
    \end{align*}
    Thus, we deduce that
    \begin{align*}
        &\frac{1}{2}\|B(u^{k+1}-u^{k})\|_{2}^{2} + \frac{\lambda_1}{2\mu_1}\|d_{w}^{k} - \nabla_{w}u^{k+1}\|_{2}^{2} + \frac{\lambda_2}{2\mu_2}\|d_{x}^{k} - \nabla_{x}u^{k+1} + \nabla_{x}u^{k}\|_{2}^{2} \\
        &+\frac{\lambda_2}{2\mu_2}\|d_{y}^{k} - \nabla_{y}u^{k+1} + \nabla_{y}u^{k}\|_{2}^{2} \leq \gamma_{k}, k=1,2,\ldots
    \end{align*}
    where
    \begin{align}\label{gamma}
        \gamma_{k} = \frac{\lambda_1}{2\mu_1}\|d_{w}^{k} - \nabla_{w}u^{k}\|_{2}^{2} + \frac{\lambda_2}{2\mu_2}\|d_{x}^{k} - \nabla_{x}u^{k} + \nabla_{x}u^{k-1}\|_{2}^{2}
        +\frac{\lambda_2}{2\mu_2}\|d_{y}^{k} - \nabla_{y}u^{k} + \nabla_{y}u^{k-1}\|_{2}^{2}
    \end{align}
    This together with (\ref{proof_2_4}) $\sim$ (\ref{proof_2_6}) gives:
    \begin{align}
        \frac{1}{2}\|B(u^{k+1}-u^{k})\|_{2}^{2} + \gamma_{k+1} \leq \gamma_{k}
    \end{align}
    This indicates that $\gamma_{1} \geq \gamma{2} \geq \ldots$ and $\gamma_{k} \geq 0$ for all $k$. Therefore, $\lim_{k\rightarrow \infty} \gamma_{k}$ exists. Since $B$ is positive definite, we can achieve that $\lim_{k\rightarrow \infty} (u^{k+1} - u^{k}) = 0$.
\end{proof}

\begin{lemma} \label{lemma_3}
  For $k=1,2,\ldots$, let $d_{w}^{k}, d_{x}^{k}, d_{y}^{k}, b_{w}^{k}, b_{x}^{k},b_{y}^{k}$ and $u^{k+1}$ be given by the iteration (\ref{new_step1}) to (\ref{new_step4}). Then all the sequences $(d_{w}^{k})_{k=1,2,\ldots},(d_{x}^{k})_{k=1,2,\ldots},(d_{y}^{k})_{k=1,2,\ldots},(b_{w}^{k})_{k=1,2,\ldots},(b_{x}^{k})_{k=1,2,\ldots},(b_{y}^{k})_{k=1,2,\ldots}$ and $(u^{k})_{k=1,2,\ldots}$ are bounded. Moreover,
  \begin{align*}\label{}
    u^{k+1} = v + \frac{\lambda_1}{\mu_1} div_{w} b_{w}^{k} - \frac{\lambda_2}{\mu_2} (\nabla_{x}^{T} b_{x}^{k} + \nabla_{y}^{T} b_{y}^{k})
  \end{align*}
\end{lemma}

\begin{proof}
  It is proved that $\|b_{w}^{k}\|_{\infty} \leq \mu_1/\lambda_1, \|b_{x}^{k}\|_{\infty} \leq \mu_2/\lambda_2$ and $\|b_{y}^{k}\|_{}\infty \leq \mu_2/\lambda_2$. Differentiate the right part of the (\ref{equiv_problem7}) and set to zero, we have
  \begin{align*}
        &B^{2}(u-v) - B^{2}(u^{k}-v) + \frac{\lambda_1}{\mu_1} \text{div}_{w}(d_{w}^{k} - \nabla_{w}u)\\
        &-\frac{\lambda_2}{\mu_2} \nabla_{x}^{T}(d_{x}^{k} - \nabla_{x}u + \nabla_{x}u_p)
        -\frac{\lambda_2}{\mu_2} \nabla_{y}^{T}(d_{y}^{k} - \nabla_{y}u + \nabla_{y}u_p) = 0
  \end{align*}
  Choosing $u = u^{k+1}$, the above equation can be reformulated as
  \begin{align*}
        &B^{2}(u^{k+1}-u^{k}) -2\frac{\lambda_1}{\mu_1} \Delta_{w}u^{k+1} - \frac{\lambda_2}{\mu_2} \Delta u^{k+1}
        = \frac{\lambda_2}{\mu_2} (\nabla_{x}^{T}d_{x}^{k} + \nabla_{x}^{T}d_{y}^{k}) - \frac{\lambda_1}{\mu_1} \text{div}_{w} d_{w}^{k}
        -\frac{\lambda_2}{\mu_2} \Delta u_p
  \end{align*}
  Since $B^{2} = (I+2\lambda \Delta_{w}+\lambda \Delta)$, we have
  \begin{align*}
        &(u^{k+1}-u^{k})
        = \frac{\lambda_2}{\mu_2} (\nabla_{x}^{T}d_{x}^{k} + \nabla_{x}^{T}d_{y}^{k}) + \frac{\lambda_2}{\mu_2} \Delta u^{k} - \frac{\lambda_2}{\mu_2} \Delta u_p
        - \frac{\lambda_1}{\mu_1} \text{div}_{w} d_{w}^{k} + 2\frac{\lambda_1}{\mu_1} \Delta_{w} u^{k}
  \end{align*}
  Using (\ref{equiv_problem4}) to (\ref{equiv_problem6}), we can reformulate the above equation as follows
  \begin{align*}
        (u^{k+1}-u^{k})
        = \frac{\lambda_2}{\mu_2} \nabla_{x}^{T}(b_{x}^{k-1} - b_{x}^{k}) + \frac{\lambda_2}{\mu_2} \nabla_{y}^{T}(b_{y}^{k-1} - b_{y}^{k})
        - \frac{\lambda_1}{\mu_1} \text{div}_{w} (b_{w}^{k-1} - b_{w}^{k})
  \end{align*}
  It follows that
  \begin{align*}
        \sum_{k=1}^{n} (u^{k+1}-u^{k})
        = \sum_{k=1}^{n} [\frac{\lambda_2}{\mu_2} \nabla_{x}^{T}(b_{x}^{k-1} - b_{x}^{k}) + \frac{\lambda_2}{\mu_2} \nabla_{y}^{T}(b_{y}^{k-1} - b_{y}^{k})
        - \frac{\lambda_1}{\mu_1} \text{div}_{w} (b_{w}^{k-1} - b_{w}^{k})]
  \end{align*}
  Since $u^{1} = v$, the above equation reduce to
  \begin{align*}
        u^{n+1}-v
        = -\frac{\lambda_2}{\mu_2} \nabla_{x}^{T}b_{x}^{n} - \frac{\lambda_2}{\mu_2} \nabla_{y}^{T}b_{y}^{n} + \frac{\lambda_1}{\mu_1} \text{div}_{w}b_{w}^{n}
  \end{align*}
  Hence, we can get the (\ref{new_step4}). Since $\|b_{w}^{k}\|_{\infty}, \|b_{x}^{k}\|_{\infty}$ and $\|b_{y}^{k}\|_{\infty}$ are bounded, we have the sequence $(u^{k})_{k=1,2,\ldots}$ is bounded. Moreover, by (\ref{equiv_problem4}) to (\ref{equiv_problem6}) the sequence $(d_{w}^{k})_{k=1,2,\ldots}, (d_{x}^{k})_{k=1,2,\ldots}$ and $(d_{y}^{k})_{k=1,2,\ldots}$ are also bounded.
\end{proof}

\begin{lemma}\label{lemma_4}
  For $k=1,2,\ldots$, let $d_{w}^{k}, d_{x}^{k}, d_{y}^{k}, b_{w}^{k}, b_{x}^{k},b_{y}^{k}$ and $u^{k+1}$ be given by the iteration (\ref{new_step1}) to (\ref{new_step4}). Then
  \begin{align*}\label{}
    &\lim_{k\rightarrow \infty} (d_{w}^{k} - \nabla_{w} u^{k}) = 0\\
    &\lim_{k\rightarrow \infty} (d_{x}^{k} - \nabla_{x} (u^{k} - u_p)) = 0\\
    &\lim_{k\rightarrow \infty} (d_{y}^{k} - \nabla_{y} (u^{k} - u_p)) = 0
  \end{align*}
\end{lemma}

\begin{proof}
  For $G(d_{x})=\|d\|_{1}$, since $g_{x}^{k} = \frac{\lambda_2}{\mu_2} b_{x}^{k}\in \partial G(d_{x}^{k}), k=1,2,\ldots$, we can get
  \begin{align*}\label{}
    &[G(d_{x}) - G(d_{x}^{k+1}) - \langle g_{x}^{k+1}, d_{x}-d_{x}^{k+1} \rangle]
    - [G(d_{x}) - G(d_{x}^{k}) - \langle g_{x}^{k}, d_{x}-d_{x}^{k} \rangle]\\
    &+ [G(d_{x}^{k+1}) - G(d_{x}^{k}) - \langle g_{x}^{k}, d_{x}^{k+1}-d_{x}^{k} \rangle] = \langle g_{x}^{k}-g_{x}^{k+1}, d - d_{x}^{k+1} \rangle
  \end{align*}
  Since $G(d_{x}^{k+1}) - G(d_{x}^{k}) - \langle g_{x}^{k}, d_{x}^{k+1}-d_{x}^{k}\rangle \geq 0$, we have
  \begin{equation}\label{proof_4_1}
    \langle g_{x}^{k}-g_{x}^{k+1}, d - d_{x}^{k+1} \rangle \geq -\|d_{x}^{k+1}\|_{1} - \langle g_{x}^{k+1}, d - d_{x}^{k+1} \rangle + \|d_{x}^{k}\|_{1} + \langle g_{x}^{k}, d - d_{x}^{k} \rangle
  \end{equation}
  Choosing $u_p = u^{k}$, the (\ref{equiv_problem5}) indicates that $g_{x}^{k} - g_{x}^{k+1} = \frac{\lambda_2}{\mu_2} (d_{x}^{k+1} - \nabla_{x}u^{k+1} + \nabla_{x}u^{k})$. Thus, by Bregman Distance we have
  \begin{align}\label{proof_4_2}
    \frac{\lambda_2}{2\mu_2}\| d_{x}-\nabla_{x}u^{k+1}+\nabla_{x}u^{k} \|_{2}^{2} + \frac{\lambda_2}{2\mu_2}\| d_{x}^{k+1}-\nabla_{x}u^{k+1}+\nabla_{x}u^{k} \|_{2}^{2} - \langle g_{x}^{k}-g_{x}^{k+1}, d_{x} - d_{x}^{k+1} \rangle \geq 0
  \end{align}
  From (\ref{proof_4_1}) and (\ref{proof_4_2}), we have
  \begin{align*}
    \frac{\lambda_2}{2\mu_2}\| d_{x}-\nabla_{x}u^{k+1}+\nabla_{x}u^{k} \|_{2}^{2} + \frac{\lambda_2}{2\mu_2}\| d_{x}^{k+1}-\nabla_{x}u^{k+1}+\nabla_{x}u^{k} \|_{2}^{2} \geq \\
    -\|d_{x}^{k+1}\|_{1} - \langle g_{x}^{k+1}, d_{x} - d_{x}^{k+1} \rangle + \|d_{x}^{k}\|_{1} + \langle g_{x}^{k}, d_{x} - d_{x}^{k} \rangle
  \end{align*}
  Choosing $d_{x} = \nabla_{x}u^{k+1} - \nabla_{x}u^{k}$ and defining $\beta_{k} = \langle g_{x}^{k}, \nabla_{x}u^{k} - \nabla_{x}u^{k-1} - d_{x}^{k} \rangle$, we have
  \begin{align*}
    \frac{\lambda_2}{2\mu_2}\| d_{x}^{k+1}-\nabla_{x}(u^{k+1}-u^{k}) \|_{2}^{2} \leq
    (\|d_{x}^{k+1}\|_{1}-\|d_{x}^{k}\|_{1}) + (\beta_{k+1}-\beta_{k})\\
     - \langle g_{x}^{k}, (\nabla_{x}u^{k+1}-\nabla_{x}u^{k})-(\nabla_{x}u^{k}-\nabla_{x}u^{k-1})\rangle
  \end{align*}
  Therefore, for $1\leq m \leq n$, we can get
  \begin{align*}
    \sum_{k=m}^{n-1}\frac{\lambda_2}{2\mu_2}\| d_{x}^{k+1}-\nabla_{x}(u^{k+1}-u^{k}) \|_{2}^{2} \leq
    \sum_{k=m}^{n-1}[(\|d_{x}^{k+1}\|_{1}-\|d_{x}^{k}\|_{1}) + (\beta_{k+1}-\beta_{k})\\
     - \langle g_{x}^{k}, (\nabla_{x}u^{k+1}-\nabla_{x}u^{k})-(\nabla_{x}u^{k}-\nabla_{x}u^{k-1})\rangle]
  \end{align*}
  Then
  \begin{align*}
    \sum_{k=m}^{n-1}\frac{\lambda_2}{2\mu_2}\| d_{x}^{k+1}-\nabla_{x}(u^{k+1}-u^{k}) \|_{2}^{2} \leq
    (\|d_{x}^{n}\|_{1}-\|d_{x}^{m}\|_{1}) + (\beta_{n}-\beta_{m})\\
     - \sum_{k=m}^{n-1}[\langle g_{x}^{k},(\nabla_{x}u^{k+1}-\nabla_{x}u^{k})-(\nabla_{x}u^{k}-\nabla_{x}u^{k-1})\rangle]
  \end{align*}
  Recall Lemma \ref{lemma_3}, the sequences $(d_{x}^{k})_{k=1,2,\ldots},(g_{x}^{k})_{k=1,2,\ldots},(u^{k})_{k=1,2,\ldots}$ are bounded, there exist positive constants $C_{1}$ and $C_{2}$ independent of $n$ and $m$ such that
  \begin{align}
    \sum_{k=m}^{n-1}\frac{\lambda_2}{2\mu_2}\| d_{x}^{k+1}-\nabla_{x}(u^{k+1}-u^{k}) \|_{2}^{2} \leq C_1 + C_2 (n-m)\eta_{m} \label{add1}
  \end{align}
  where $\eta_{m}:= \sup_{m\leq k}\|u^{k+1} - u^{k}\|_{2}$. Combining the Lemma \ref{lemma2}, we have $\lim_{m\rightarrow \infty} \eta_{m} = 0$.
  The same way to $d_{w}$ and $d_{y}$:
  \begin{align}
    \sum_{k=m}^{n-1}\frac{\lambda_1}{2\mu_1}\| d_{w}^{k+1}-\nabla_{w}u^{k+1} \|_{2}^{2} \leq C_1 + C_2 (n-m)\eta_{m} \label{add2}\\
    \sum_{k=m}^{n-1}\frac{\lambda_2}{2\mu_2}\| d_{y}^{k+1}-\nabla_{y}(u^{k+1}-u^{k}) \|_{2}^{2} \leq C_1 + C_2 (n-m)\eta_{m} \label{add3}
  \end{align}
  Adding (\ref{add1}), (\ref{add2}) and (\ref{add3}) gives
  \begin{align*}
    \sum_{k=m}^{n-1}\gamma_{k+1} \leq 3C_1 + 3C_2 (n-m)\eta_{m} \label{add2}
  \end{align*}
  Since $\gamma_{n}\leq \gamma_{k}$ for $k\leq n$, $(n-m)\gamma_{n}\leq 3C_1+3C_2(n-m)$, that means
  \begin{align*}
    \gamma_{n} \leq \frac{3C_1}{n-m} + 3C_2\eta_{m}
  \end{align*}
  The above inequality indicates that $\lim_{n\rightarrow \infty} \gamma_{n}=0$
\end{proof}

\begin{thmA}
\label{main_theorem}
For $k=1,2,\ldots$, let $b_{w}^{k}, b_{x}^{k},b_{y}^{k}$ and $u^{k+1}$ be given by the iteration (\ref{new_step1}) to (\ref{new_step4}). If $0< 20\lambda_1 + 4\lambda_2 <1$, then $\lim_{k\rightarrow \infty} u^{k} = u^{\ast}$.
\end{thmA}

\begin{proof}
  Let $F(u):= (1/2) \|u-v\|_{2}^{2}$, then $\partial F(u) = u-v$. For $w\in \mathbb{R}^{N^{2}}$ we can get
  \begin{equation}\label{}
    F(u^{k+1} + w) - F(u^{k+1}) - \langle u^{k+1}-v, w\rangle \geq 0
  \end{equation}
  From Lemma \ref{lemma_3}, we have $-(u^{k+1}-v) = \frac{\lambda_2}{\mu_2} \nabla_{x}^{T}b_{x}^{k} + \frac{\lambda_2}{\mu_2} \nabla_{y}^{T}b_{y}^{k} - \frac{\lambda_1}{\mu_1} \text{div}_{w}b_{w}^{k}$, and moreover $\langle div_{w}p, q \rangle = -\langle p, \nabla_{w} q \rangle$. Then,
  \begin{equation}\label{proof_main_1}
    F(u^{k+1} + w) - F(u^{k+1}) + \langle \frac{\lambda_2}{\mu_2} b_{x}^{k}, \nabla_{x}w \rangle + \langle \frac{\lambda_2}{\mu_2} b_{y}^{k}, \nabla_{y}w \rangle + \langle \frac{\lambda_1}{\mu_1} b_{w}^{k}, \nabla_{w}w \rangle\geq 0
  \end{equation}
  Recall that $G(d) = \|d\|_{1}$, $\frac{\lambda_1}{\mu_1} b_{w}^{k}\in \partial G(d_{w}^{k}), \frac{\lambda_2}{\mu_2} b_{x}^{k}\in \partial G(d_{x}^{k})$ and $\frac{\lambda_2}{\mu_2} b_{y}^{k}\in \partial G(d_{y}^{k})$, so,
  \begin{align}
     \|d_{w}^{k} + \nabla_{w}w\|_{1} - \|d_{w}^{k}\|_{1} - \langle \frac{\lambda_1}{\mu_1} b_{w}^{k}, \nabla_{w}w \rangle\geq 0 \label{proof_main_2}\\
     \|d_{x}^{k} + \nabla_{x}w\|_{1} - \|d_{x}^{k}\|_{1} - \langle \frac{\lambda_2}{\mu_2} b_{x}^{k}, \nabla_{x}w \rangle\geq 0 \label{proof_main_3}\\
     \|d_{y}^{k} + \nabla_{y}w\|_{1} - \|d_{y}^{k}\|_{1} - \langle \frac{\lambda_2}{\mu_2} b_{y}^{k}, \nabla_{y}w \rangle\geq 0 \label{proof_main_4}
  \end{align}
  Adding (\ref{proof_main_1}) $\sim$ (\ref{proof_main_4}) gives
  \begin{equation}\label{proof_main_5}
     \|d_{w}^{k}\|_{1} + \|d_{x}^{k}\|_{1} + \|d_{y}^{k}\|_{1} + F(u^{k+1}) \leq
     \|d_{w}^{k} + \nabla_{w}w\|_{1} + \|d_{x}^{k} + \nabla_{x}w\|_{1} + \|d_{y}^{k} + \nabla_{y}w\|_{1} + F(u^{k+1} + w)
  \end{equation}
  Suppose that $(k_j)_{j=1,2,\ldots}$ is an increasing sequence of positive integers such that the sequence $(u^{k_j})_{j=1,2,\ldots}$ converges to the limit $\tilde{u}$. By Lemma \ref{lemma2}, we have $\lim_{k\rightarrow \infty} (u^{k+1} - u^{k}) = 0$. Therefore, $\lim_{j\rightarrow \infty} u^{k_j+1} = \tilde{u}$. Moreover, we have
  the following via the Lemma \ref{lemma_4}:
  \begin{equation}\label{}
     \lim_{j\rightarrow \infty} d_w^{k} = \lim_{j\rightarrow \infty} [(d_w^{k} - \nabla_w u^{k}) + \nabla_w u^{k}]
      = \nabla_w \tilde{u}
  \end{equation}
  also have
  \begin{equation}\label{}
     \lim_{j\rightarrow \infty} d_x^{k} = \lim_{j\rightarrow \infty} [(d_x^{k} - \nabla_x (u^{k} - u_p)) + \nabla_x (u^{k} - u_p)]
      = \nabla_x (\tilde{u} - u_p)
  \end{equation}
  \begin{equation}\label{}
     \lim_{j\rightarrow \infty} d_y^{k} = \lim_{j\rightarrow \infty} [(d_y^{k} - \nabla_y (u^{k} - u_p)) + \nabla_y (u^{k} - u_p)]
      = \nabla_y (\tilde{u} - u_p)
  \end{equation}
  Replacing $k$ by $k_j$ in (\ref{proof_main_5}) and let $j\rightarrow \infty$, we have:
  \begin{align*}
    \|\nabla_{w} \tilde{u}\|_{1} + \|\nabla_{x} (\tilde{u} - u_p)\|_{1} + \|\nabla_{y} (\tilde{u} - u_p)\|_{1} + F(\tilde{u}) \leq \|\nabla_{w} (\tilde{u} + w)\|_{1} + \\
    \|\nabla_{x} (\tilde{u} - u_p + w)\|_{1} + \|\nabla_{y} (\tilde{u} - u_p + w)\|_{1} + F(\tilde{u} + w)
  \end{align*}
  The above equations hold for all the $w\in \mathbb{R}^{N^{2}}$. On the other hand, $u^{\ast}$ is the unique solution to the minimization problem (\ref{split_bregman_rof1}). Therefore, we must have $\tilde{u} = u^{\ast}$. Since $(u^{k})_{k=1,2,\ldots}$ is a bounded sequence, we have
  \begin{equation}\label{}
    \lim_{k\rightarrow \infty} u^{k} = u^{\ast}.
  \end{equation}
  This completes the proof of the Main Theorem 1.

\end{proof}
